\documentclass[10pt]{article} % added twocolumn
\usepackage{amsmath, amssymb, bm}
\IfFileExists{placeins.sty}{\usepackage{placeins}}{\newcommand{\FloatBarrier}{}}
\allowdisplaybreaks[4]
\usepackage{graphicx}
\IfFileExists{psfrag.sty}{\usepackage{psfrag}}{}
\IfFileExists{epsf.sty}{\usepackage{epsf}}{}
\usepackage{enumerate}
\usepackage{comment}
\usepackage{natbib}
\usepackage{xcolor}
\usepackage{url} % for URLs
\usepackage{algorithm}
\usepackage{algpseudocode}
\usepackage{amsthm}
\usepackage{amsfonts} %
\usepackage[colorlinks=true, citecolor=blue]{hyperref}

\IfFileExists{tikz.sty}{%
	\usepackage{tikz}
	\usetikzlibrary{decorations.markings,arrows.meta}
}{%
	\def\draw##1;{}%
	\def\fill##1;{}%
	\def\filldraw##1;{}%
	\def\node##1;{}%
	\def\coordinate##1;{}%
	\def\foreach##1in##2##3{}%
}
\IfFileExists{subcaption.sty}{\usepackage{subcaption}}{}
\theoremstyle{plain}
\newtheorem{theorem}{Theorem}
\newtheorem{lemma}[theorem]{Lemma}
\newtheorem{proposition}[theorem]{Proposition}

\theoremstyle{definition}

\theoremstyle{remark}

\IfFileExists{titlesec.sty}{\usepackage{titlesec}}{} % optional, to adjust section spacing if needed
\IfFileExists{enumitem.sty}{\usepackage{enumitem}}{\newcommand{\setlist}[2][]{}} % for description spacing
\IfFileExists{microtype.sty}{\usepackage{microtype}}{} % subtle kerning/spacing fixes
\allowdisplaybreaks[4] % allow equations to break across lines/pages
\setlist[description]{leftmargin=1em,labelindent=0em}
\newcommand{\blind}{0}

\def\spacingset#1{\renewcommand{\baselinestretch}{#1}\small\normalsize}
\spacingset{1}

\begin{document}
	
	\if0\blind
	{
		\title{\bf Adaptive Iterative Hard Thresholding for Online High-dimensional Quantile Regression}
		\author{Zitian Zhou, Nan Lin \\ Washington University in St. Louis  }
		\date{}
		\maketitle
	} \fi
	
	\if1\blind
	{
		\bigskip\bigskip\bigskip
		\begin{center}
			{\LARGE\bf Adaptive Iterative Hard Thresholding for Online High-dimensional Quantile Regression}
		\end{center}
		\medskip
	} \fi
	
	\begin{abstract}
Online high-dimensional regression requires algorithms that can update sequentially while preserving structural sparsity. We propose \textit{Adaptive Iterative Hard Thresholding (AIHT)}, an online sparse-regression framework that alternates stochastic subgradient updates with adaptively scheduled hard-thresholding steps. The key idea is to separate support discovery from local refinement: early in the learning process, AIHT delays thresholding so that weak but informative coordinates have time to accumulate signal, while later it increases the projection frequency to stabilize the sparse estimator and exploit local curvature. We develop the theory for high-dimensional online quantile regression, a challenging setting in which the loss is nonsmooth and the data may exhibit heterogeneity or heavy-tailed noise. Under restricted curvature and gradient-leakage conditions, AIHT remains in an inflated sparse cone, exhibits a two-phase convergence behavior, and attains logarithmic regret for the sliding-window objective. Simulations for online quantile regression, together with threshold-scheduling ablations, support the proposed mechanism and illustrate its advantage over standard online sparse-learning baselines.
\end{abstract}
	
	\noindent%
	
	\section{Introduction}
	\label{intro}
	Many modern machine learning applications operate in streaming environments, where data arrive sequentially, either as individual observations or in mini-batches \citep{zinkevich2003online, shalevshwartz2012online}. In such settings, the number of predictors is often large relative to the sample size available at each time step. This makes high-dimensional regression a fundamental problem in sequential learning (see, e.g., \citealp{buhlmann2011statistics, hastie2015statistical, negahban2012unified, zinkevich2003online, shalevshwartz2012online}). In these regimes, it is important not only to impose structural assumptions such as sparsity, but also to develop methods that remain stable under noise, adapt to heteroscedasticity, and are robust to potential distributional shifts \citep{buhlmann2011statistics, negahban2012unified, hastie2015statistical, koenker1978regression, koenker2005quantile}.
	
	Consider a streaming model in which data $(Y_t, X_t) \in \mathbb{R} \times \mathbb{R}^p$ arrive sequentially over time $t = 1, \dots, T$, where the time horizon $T$ may grow without bound. We consider the sparse linear regression model
	\begin{equation}
		\label{model}
		Y_t = X_t^\top \beta^* + \varepsilon_t,
	\end{equation}
	where the true coefficient vector $\beta^* \in \mathbb{R}^p$ is sparse and satisfies \(\|\beta^\ast\|_0=s_0\ll p\); here $\|\cdot\|_0$ denotes the number of nonzero entries. Our goal is to estimate $\beta^*$ iteratively by minimizing a sequence of convex or locally convex loss functions $\ell_t(\beta)$. Quantile regression is the main theoretical case study in this paper, but the algorithmic construction is not tied to the check loss. A standard approach in online learning is stochastic gradient descent (SGD), which updates the parameter according to
	\begin{equation}
		\beta_{t+1} = \beta_t - \eta_t g_t(\beta_t),
	\end{equation}
	where $\eta_t$ is the step size and $g_t(\beta_t) \in \partial \ell_t(\beta_t)$ is a (sub)gradient of the loss at time $t$. The performance of an online algorithm is typically measured by \emph{regret}, defined as
	\begin{equation}
		R_T = \sum_{t=1}^T \bigl( \ell_t(\beta_t) - \ell_t(\beta^*) \bigr),
	\end{equation}
	which represents the cumulative difference between the incurred loss and that of the best fixed parameter in hindsight. In low-dimensional settings, early work on online learning for model \eqref{model} (e.g., \citealp{zinkevich2003online, shalevshwartz2012online}) primarily focused on Gaussian models with smooth and strongly convex loss functions. Classical results in online convex optimization show that SGD achieves $O(\sqrt{T})$ regret for general convex losses and $O(\log T)$ regret when the loss is strongly convex with appropriately chosen step sizes \citep{zinkevich2003online, hazan2007log, shalevshwartz2012online}.
	
	Extending such results to high-dimensional settings with potentially unstable noise introduces substantial challenges. On the one hand, when the dimensionality $p$ exceeds the effective sample size, the empirical Hessian matrix becomes singular, making global strong convexity unrealistic \citep{buhlmann2011statistics, negahban2012unified}. To address this issue, it is standard to impose sparsity assumptions and incorporate regularization techniques \citep{buhlmann2011statistics, negahban2012unified}. Popular approaches include soft-thresholding methods such as $\ell_1$ (Lasso) penalties \citep{tibshirani1996regression}, as well as non-convex penalties such as SCAD \citep{fan2001variable} and MCP \citep{zhang2010nearly}, and hard thresholding methods that enforce explicit $\ell_0$ constraints. While $\ell_1$ and $\ell_2$ regularization preserve convexity of the loss function, non-convex approaches such as hard thresholding or SCAD/MCP introduce additional analytical challenges and make achieving optimal regret more difficult. For existing work on smooth and strongly convex losses with soft regularization, we refer to \citep{buhlmann2011statistics, hastie2015statistical, negahban2012unified, bickel2009simultaneous}. On the other hand, the choice of loss function plays a critical role in online learning. Despite their computational and theoretical convenience, smooth convex losses such as the squared error are known to be sensitive to heavy-tailed noise, heteroscedasticity, and distributional shifts. These limitations motivate the use of more robust alternatives, among which quantile regression provides a natural and flexible framework for modeling heterogeneous noise and asymmetric distributions \citep{koenker1978regression, koenker2005quantile, koenker2001quantile}.
	
	These considerations motivate the study of high-dimensional quantile regression with explicit sparsity control in an online setting. In this regime, the loss function is given by the quantile (check) loss
	\begin{equation}
		\rho_\tau(u) = u\bigl(\tau - \mathbf{1}\{u < 0\}\bigr),
	\end{equation}
	which provides a robust alternative to the squared loss and allows for modeling heterogeneous and asymmetric noise.  A natural way to incorporate hard sparsity into online optimization is to combine stochastic gradient descent (SGD) with hard thresholding, leading to the classical Iterative Hard Thresholding (IHT) update,
	\begin{equation}
		\beta_{t+1} = H_s\!\left(\beta_t - \eta_t g_t(\beta_t)\right),
	\end{equation}
	where $\eta_t$ is the step size, $H_s(\cdot)$ retains the largest $s$ coordinates in magnitude, and $g_t(\beta) = -X_t\bigl(\tau - \mathbf{1}\{Y_t \le X_t^\top \beta\}\bigr)$ is the subgradient.
	
	However, naively applying hard thresholding after every stochastic update can introduce important difficulties. Since the projection retains only the currently largest coordinates, a true variable with a small current magnitude may be removed before it has accumulated enough signal to enter the active set. This makes support discovery especially fragile when the step size is decaying or when the signal is initially weak. A common remedy is to apply thresholding only periodically, as in \citep{langford2009sparse}, so that coordinates can evolve for several gradient steps before the next sparse projection. Periodic thresholding can mitigate the one-step entry barrier, but it introduces a new scheduling problem: the interval between projections must be long enough to allow informative coordinates to emerge, yet short enough to prevent noise-driven off-support coordinates from accumulating without control \citep{agarwal2012fast, jain2014iterative}. Moreover, effective high-dimensional analysis typically relies on keeping the iterates inside a restricted region where sufficient curvature is present \citep{negahban2012unified, agarwal2012fast}. If thresholding is too infrequent, the unprojected updates may leave this region; if it is too frequent, support discovery may be suppressed. These competing requirements highlight the need for an adaptive thresholding mechanism that balances sparsity enforcement with stable optimization dynamics.
	
	\begin{figure*}[htp]
		\centering
		\includegraphics[width=0.98\linewidth]{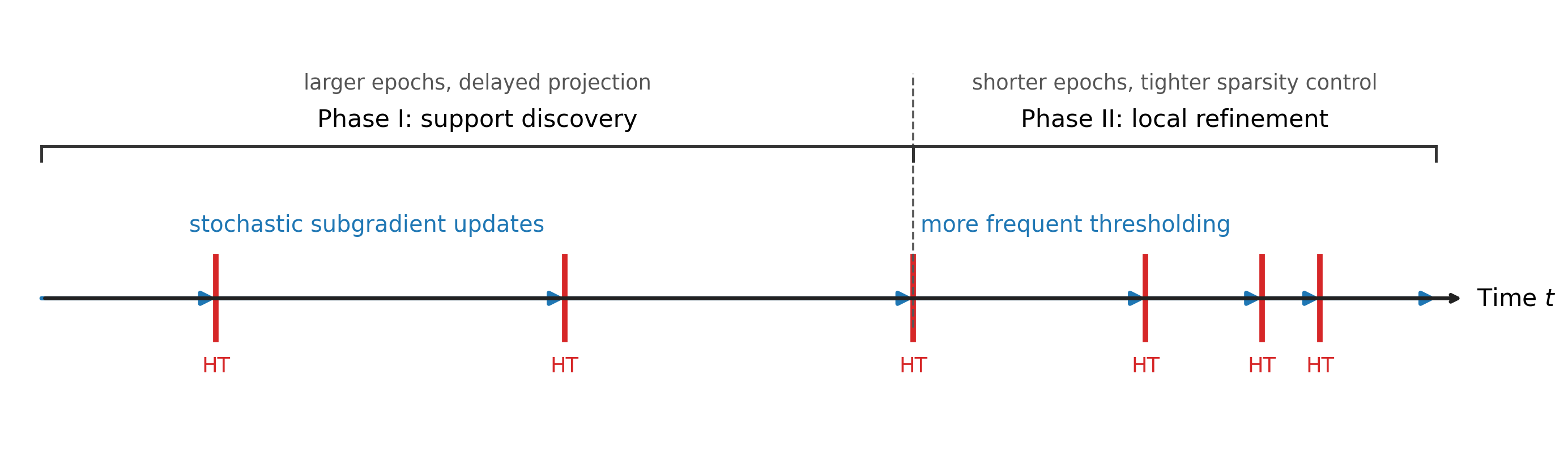}
		\caption{Two-layer structure of AIHT, where hard-thresholding (HT) is applied intermittently with decreasing intervals.}
		\label{fig:ht-schedule}
	\end{figure*}
	
	Based on these considerations, we propose the Adaptive Iterative Hard Thresholding (AIHT) algorithm, which alternates between stochastic subgradient descent updates and intermittent hard thresholding under an adaptive schedule. By monitoring the stability of SGD updates through gradient mappings, we design theoretically grounded step sizes $\eta_t$ and thresholding intervals $k_t$. The algorithm operates in two phases. In Phase I, it employs less frequent thresholding together with relatively larger step sizes, allowing the iterates sufficient time to grow and enabling relevant coordinates to emerge. In Phase II, the algorithm transitions to more frequent thresholding and smaller step sizes, promoting sparsification and refining the estimator; Figure~\ref{fig:ht-schedule} illustrates this schedule. Figure~\ref{fig:comparison} highlights the importance of controlling the thresholding frequency: AIHT converges to a more accurate solution, whereas thresholding at every iteration or at fixed intervals can lead to suboptimal convergence or persistent oscillations. In the first phase, the estimation error decreases rapidly as the algorithm identifies and stabilizes relevant coordinates. After the transition to the refinement phase, AIHT produces more accurate updates and achieves a fast contraction rate of $O((\log t)/t)$.These results yield logarithmic refinement-phase regret for the sliding-window objective; the total regret is \(C_{\mathrm I}(T_0)+C_{\mathrm{burn}}+O(\log T)\), and becomes \(O(\log T)\) when the transition and burn-in costs are bounded independently of \(T\).
	
	\begin{figure}
		\centering
		\includegraphics[width=1\linewidth, height=0.3\linewidth]{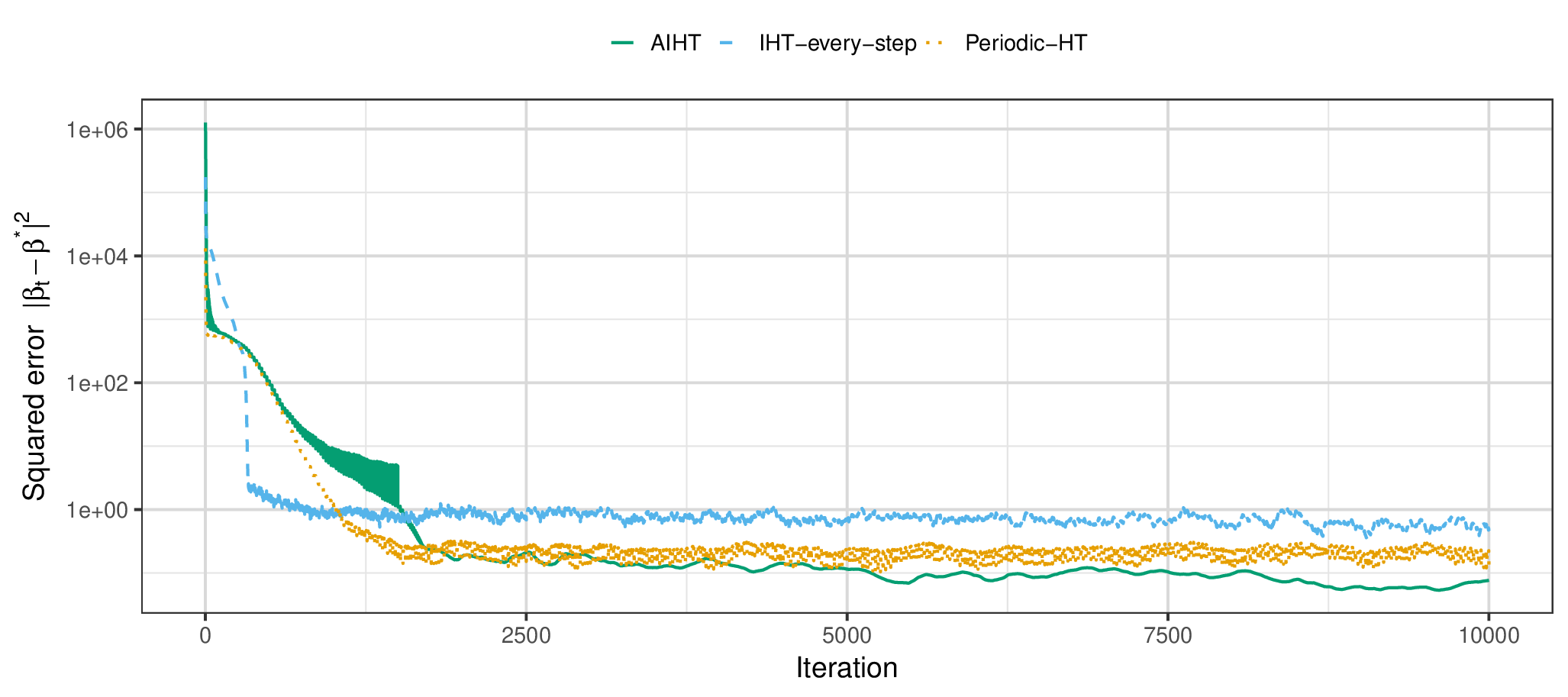}
		\caption{Error comparison among conventional IHT, which performs hard thresholding at every step, periodic IHT, and AIHT.}
		\label{fig:comparison}
	\end{figure}
	Our contributions are threefold. First, we formulate AIHT as an online sparse-regression framework and identify the epoch-mass mechanism that explains why adaptive threshold timing is necessary. Second, for high-dimensional online quantile regression, we establish cone stability, two-phase convergence, and logarithmic regret guarantees. Third, we extend the quantile-regression analysis to settings with distributional shifts and obtain analogous segmentwise guarantees. The quantile case is deliberately demanding because the check loss is nonsmooth and robust to heavy-tailed noise. In addition, our analysis provides a rigorous characterization of a two-phase convergence phenomenon for AIHT.

	\section{Methodology}
\label{sec:methodology}

In this section, we introduce the AIHT update and the notation used throughout the paper. The main algorithmic idea is to separate the stochastic-gradient evolution from the sparse projection: the iterates are allowed to evolve without thresholding within an epoch, and hard thresholding is applied only at selected epoch endpoints. This separation makes the cumulative step-size mass within an epoch, rather than only the number of iterations between projections, a central scheduling quantity.

\subsection{Notation and preliminaries}

Let \(\beta^\ast\in\mathbb{R}^p\) denote the unknown sparse target parameter, and let $S=\operatorname{supp}(\beta^\ast),  |S|=s_0,$ where \(s_0\) is the true sparsity level. The algorithm uses a working sparsity level \(s\ge s_0\). For any vector \(v\in\mathbb{R}^p\) and index set \(A\subseteq\{1,\dots,p\}\), we write \(v_A\) for the restriction of \(v\) to the coordinates in \(A\), and \(A^c\) for the complement of \(A\). Throughout, the estimation error is denoted by $\Delta_t=\beta_t-\beta^\ast.$ The hard-thresholding operator is denoted by \(H_s(\cdot)\). Given \(v\in\mathbb{R}^p\), \(H_s(v)\) retains the \(s\) largest coordinates of \(v\) in magnitude and sets all remaining coordinates to zero, with ties broken arbitrarily. Thus \(H_s(v)\) is an Euclidean projection of \(v\) onto the set of \(s\)-sparse vectors.

The AIHT template can be used with any online regression loss for which a stochastic subgradient is available and the local restricted geometry required by the analysis can be verified. The formal theory in this paper focuses on online quantile regression, but the algorithmic structure is not specific to the check loss.

To stabilize the online updates, AIHT uses a sliding window. Let \(W=W(p)\) be the prescribed buffer length. For \(t\ge 1\), define
\[
\mathcal I_t=\{\max(1,t-W+1),\ldots,t\},
\qquad
W_t=|\mathcal I_t|.
\]
For a generic loss \(\ell_i(\beta)=\ell(Y_i,X_i;\beta)\), the sliding-window empirical objective is 
\[
Q_t(\beta)
=
\frac{1}{W_t}\sum_{i\in\mathcal I_t}\ell_i(\beta).
\]
The algorithm uses a subgradient \(\bar g_t(\beta_t)\in\partial Q_t(\beta_t)\). In the theoretical analysis, we focus on the post-burn-in regime \(t\ge W\), where \(W_t=W\) and
\[
Q_t(\beta)
=
\frac{1}{W}\sum_{i=t-W+1}^{t}\ell_i(\beta).
\]

For the quantile-regression case study, we take \(\ell_i(\beta)=\rho_\tau(Y_i-X_i^\top\beta)\), where \(\rho_\tau(u)=u\{\tau-\mathbf 1(u<0)\}\). Using the standard subgradient convention for the check loss, the sliding-window subgradient is
\begin{equation}
    \label{move_grad}
    \bar g_t(\beta_t)
    =
    -\,\frac{1}{W_t}\sum_{i\in\mathcal I_t}
    X_i\bigl(\tau-\mathbf 1\{Y_i\le X_i^\top\beta_t\}\bigr).
\end{equation}
For \(t\ge W\), this reduces to
\[
\bar g_t(\beta_t)
=
-\,\frac{1}{W}\sum_{i=t-W+1}^{t}
X_i\bigl(\tau-\mathbf 1\{Y_i\le X_i^\top\beta_t\}\bigr),
\]
which is the form used in the main theoretical results.

The online trajectory is partitioned into epochs. Let \(\tau_j\) denote the starting time of epoch \(j\), let \(k_j\) denote its length, and define
\[
\tau_{j+1}=\tau_j+k_j,
\qquad
\mathcal T_j=\{\tau_j,\tau_j+1,\dots,\tau_{j+1}-1\}.
\]
For an index set \(A\subseteq\{1,\ldots,p\}\), let \(P_Av\) denote the vector that agrees with \(v\) on \(A\) and is zero on \(A^c\). In AIHT, each epoch \(j\) is associated with a candidate set \(A_j\), defined below in Section~\ref{subsec:aiht}. Once \(A_j\) is fixed, the update direction within epoch \(j\) is
\[
d_t=P_{A_j}\bar g_t(\beta_t),
\qquad t\in\mathcal T_j.
\] Within epoch \(j\), AIHT first forms the candidate-restricted direction \(d_t=P_{A_j}\bar g_t(\beta_t)\) and then computes
\[
\widetilde\beta_{t+1}=\beta_t-\eta_t d_t.
\]
If \(t<\tau_{j+1}-1\), then \(\beta_{t+1}=\widetilde\beta_{t+1}\). At the final step of the epoch, hard thresholding is applied:
\[
\beta_{\tau_{j+1}}=H_s(\widetilde\beta_{\tau_{j+1}}).
\]
Thus sparse projection is applied only at epoch endpoints. The cumulative step-size mass of epoch \(j\) is \(B_j=\sum_{t\in\mathcal T_j}\eta_t\), which will be used below to explain why adaptive threshold timing is more informative than a fixed number of iterations between projections. 
	
	\subsection{Support-entry failure of naive thresholding}
\label{subsec:support-entry-failure}

We first explain why the timing of hard thresholding matters. The basic obstruction is algorithmic: hard thresholding compares coordinates only at the moment of projection. If a true coordinate is currently absent from the active set, then it must become large enough before the next projection; otherwise it is immediately reset to zero. This can make support discovery difficult when thresholding is applied after every stochastic update.

For comparison, consider the every-step IHT update
\begin{equation}
    \label{eq:every-step-iht-method}
    u_{t+1}=\beta_t-\eta_t g_t(\beta_t),
    \qquad
    \beta_{t+1}=H_s(u_{t+1}),
\end{equation}
where \(g_t\) denotes the gradient or subgradient used at time \(t\). Let \(\widehat S_t=\operatorname{supp}(\beta_t)\). For a vector \(v\in\mathbb R^p\) and coordinate \(m\), define
\[
\lambda_s(v;m)
:=
\text{the \(s\)th largest value among }
\{|v_\ell|:\ell\ne m\}.
\]
Ignoring ties, coordinate \(m\) is retained by \(H_s(v)\) if and only if \(|v_m|>\lambda_s(v;m)\). Thus \(\lambda_s(v;m)\) is the instantaneous entry barrier faced by coordinate \(m\).

\begin{proposition}
    \label{prop:one-step-barrier}
    Consider every-step IHT \eqref{eq:every-step-iht-method}. Suppose \(m\in S\) is missing at time \(t\), so that \(m\notin\widehat S_t\) and \(\beta_{t,m}=0\). Let \(\lambda_t^{(-m)}=\lambda_s(u_{t+1};m)\). Then, ignoring ties,
    \begin{equation}
        \label{eq:one-step-entry}
        m\in\widehat S_{t+1}
        \quad\Longleftrightarrow\quad
        \eta_t |g_{t,m}(\beta_t)|>\lambda_t^{(-m)}.
    \end{equation}
    If, while \(m\notin\widehat S_t\), \(|g_{t,m}(\beta_t)|\le G_m\) and \(\lambda_t^{(-m)}\ge\lambda_0t^{-r}\) for some \(G_m,\lambda_0>0\) and \(r\ge0\), then for \(\eta_t=\alpha t^{-q}\) with \(q>r\), coordinate \(m\) cannot enter at any time
    \begin{equation}
        \label{eq:one-step-lock-time}
        t\ge
        T_{\mathrm{lock}}
        :=
        \left(\frac{\alpha G_m}{\lambda_0}\right)^{1/(q-r)} .
    \end{equation}
    In particular, if the competing active coordinates create a nonvanishing entry barrier, corresponding to \(r=0\), then every decaying step size \(\eta_t\asymp t^{-q}\) with \(q>0\) creates a finite entry deadline.
\end{proposition}

The point is that the condition \(\sum_t\eta_t=\infty\) does not by itself prevent this failure. Under every-step hard thresholding, a missing coordinate cannot retain the update mass from previous failed attempts, because it is set back to zero after each projection. Support entry is therefore reduced to a one-step comparison.

The same obstruction appears even in a noiseless two-dimensional quadratic problem. Consider
\begin{equation}
    \label{eq:toy-quadratic}
    L(\beta)
    =
    \frac12(\beta-\beta^\ast)^\top
    \Sigma(\beta-\beta^\ast),
    \qquad
    \beta^\ast=(\theta,0)^\top,
    \qquad
    \Sigma=
    \begin{pmatrix}
        1 & \rho\\
        \rho & 1
    \end{pmatrix},
\end{equation}
where \(\theta>0\) and \(0<\rho<1\). Let \(s=1\). The true support is \(\{1\}\). If the iterate is restricted to the wrong support \(\{2\}\), the best one-sparse point on that support is \(\bar\beta=(0,\rho\theta)^\top\).

\begin{proposition}
    \label{prop:wrong-fixed-point}
    For the quadratic loss \eqref{eq:toy-quadratic},
    \[
    \nabla L(\bar\beta)
    =
    (-(1-\rho^2)\theta,0)^\top,
    \]
    and hence
    \[
    \bar\beta-\eta\nabla L(\bar\beta)
    =
    \{\eta(1-\rho^2)\theta,\rho\theta\}^\top .
    \]
    If \(\eta<\rho/(1-\rho^2)\), then
    \[
    H_1\{\bar\beta-\eta\nabla L(\bar\beta)\}=\bar\beta.
    \]
    Consequently, for \(\eta_t=\alpha t^{-q}\) with \(q>0\), \(\bar\beta\) becomes an absorbing wrong-support point for every-step IHT once
    \begin{equation}
        \label{eq:wrong-fixed-time}
        t>
        \left\{
        \frac{\alpha(1-\rho^2)}{\rho}
        \right\}^{1/q}.
    \end{equation}
\end{proposition}

Propositions~\ref{prop:one-step-barrier} and \ref{prop:wrong-fixed-point} motivate delaying hard thresholding. If no projection is applied for several updates, then a missing coordinate is not judged by the single-step quantity \(\eta_t|g_{t,m}|\). Instead, its possible movement over an epoch is governed by the accumulated update mass during that epoch. Recalling the notation from above, this mass is
\[
B_j=\sum_{t\in\mathcal T_j}\eta_t.
\]
Thus the relevant scheduling variable is not only the number of iterations between projections, but the cumulative step-size mass available before the next projection.

A fixed thresholding period is a natural first attempt, but it is not scale-adaptive when the step size decays. If \(\mathcal T_j=\{\tau_j,\ldots,\tau_j+K-1\}\), then
\[
B_j
=
\sum_{t=\tau_j}^{\tau_j+K-1}\eta_t .
\]
The same value of \(K\) can correspond to very different amounts of optimization movement at different stages of the run.

\begin{proposition}
    \label{prop:fixed-period-vanishing-mass}
    Let \(\eta_t=\alpha t^{-q}\) with \(q\in(0,1]\). If a periodic hard-thresholding method uses a fixed period \(k_j\equiv K\), then
    \[
    B_j
    =
    \sum_{t=\tau_j}^{\tau_j+K-1}\eta_t
    \to 0
    \qquad
    \text{as } \tau_j\to\infty.
    \]
    If \(0<q<1\), then
    \begin{equation}
        \label{eq:fixedK-mass-asymp}
        B_j
        =
        \alpha K\tau_j^{-q}\{1+o(1)\}.
    \end{equation}
    If \(q=1\), then
    \begin{equation}
        \label{eq:fixedK-mass-q1}
        B_j
        =
        \alpha\log\left(1+\frac{K}{\tau_j}\right)+o(\tau_j^{-1})
        =
        \alpha K\tau_j^{-1}\{1+o(1)\}.
    \end{equation}
    Consequently, for any fixed positive support-entry threshold \(B_{\mathrm{entry}}>0\), fixed-period thresholding eventually satisfies \(B_j<B_{\mathrm{entry}}\).
\end{proposition}

Proposition~\ref{prop:fixed-period-vanishing-mass} says that a period that is adequate early in the run may become too short later, because the cumulative movement before projection vanishes. However, simply making epochs longer is not safe either. If too much update mass accumulates between projections, then off-support components may grow enough to leave the restricted region where high-dimensional curvature is available.

To formalize this upper constraint, define the inflated cone ratio
\[
r_t
=
\frac{\|\Delta_{t,S^c}\|_1}
{\|\Delta_{t,S}\|_1+\zeta},
\qquad \zeta\ge0.
\]
The small constant \(\zeta\) prevents degeneracy when the on-support error is close to zero. The condition \(r_t\le c\) is equivalent to \(\|\Delta_{t,S^c}\|_1\le c\{\|\Delta_{t,S}\|_1+\zeta\}\).

\begin{proposition}
    \label{prop:epoch-mass-cone-barrier}
    Consider one unprojected epoch \(\mathcal T_j=\{\tau_j,\ldots,\tau_{j+1}-1\}\), with update
    \[
    \Delta_{t+1}
    =
    \Delta_t-\eta_t g_t(\beta_t).
    \]
    Suppose that throughout the epoch the gradient satisfies
    \begin{equation}
        \label{eq:method-gradient-guard}
        \|g_{t,S^c}(\beta_t)\|_1
        \le
        L_{\mathrm{off}}\{\|\Delta_{t,S}\|_1+\zeta\},
        \qquad
        \|g_{t,S}(\beta_t)\|_1
        \le
        L_{\mathrm{on}}\{\|\Delta_{t,S}\|_1+\zeta\},
    \end{equation}
    and \(\eta_tL_{\mathrm{on}}\le1/2\). If \(r_{\tau_j}\le c_{\mathrm{small}}\), then, as long as \(r_t\le c_{\mathrm{large}}\),
    \begin{equation}
        \label{eq:ratio-recursion-method}
        r_{t+1}
        \le
        r_t
        +
        2\eta_t\{L_{\mathrm{off}}+c_{\mathrm{large}}L_{\mathrm{on}}\}.
    \end{equation}
    Consequently, a sufficient cone-stability condition is
    \begin{equation}
        \label{eq:Bcone-method}
        B_j
        :=
        \sum_{t\in\mathcal T_j}\eta_t
        \le
        B_{\mathrm{cone}}
        :=
        \frac{c_{\mathrm{large}}-c_{\mathrm{small}}}
        {2\{L_{\mathrm{off}}+c_{\mathrm{large}}L_{\mathrm{on}}\}} .
    \end{equation}

    Conversely, this type of upper bound cannot be removed in general. Even when \(L_{\mathrm{on}}=0\), if
    \[
    B_j>\frac{c_{\mathrm{large}}-c_{\mathrm{small}}}{L_{\mathrm{off}}},
    \]
    there exists a sequence of gradients satisfying \eqref{eq:method-gradient-guard} for which an iterate starting in the small cone exits the large cone during the epoch.
\end{proposition}

Taken together, the preceding propositions show that threshold timing is governed by two competing requirements. The epoch mass must be large enough to allow missing coordinates to accumulate signal before the next projection, but small enough to keep off-support leakage under control:
\[
B_{\mathrm{entry}}
\ \lesssim\
B_j
\ \lesssim\
B_{\mathrm{cone}}.
\]
This is the scheduling principle behind AIHT. During support discovery, thresholding should not be so frequent that it repeatedly erases weak true coordinates. During local refinement, thresholding should become more frequent so that the iterates remain close to the sparse region where restricted curvature can be exploited.
	
\subsection{Sparse cone and restricted strong convexity}
\label{subsec:cone-rsc}

The previous subsection shows that the timing of hard thresholding should be governed by the cumulative step-size mass within an epoch. We now explain why the iterates must also remain in a restricted sparse region. In high-dimensional regression, the sliding-window loss is not globally strongly convex when the ambient dimension \(p\) is larger than the effective sample size \(W\). Curvature can only be expected along directions compatible with the sparse structure of the target parameter.

Let \(\Delta=\beta-\beta^\ast\). For a constant \(c>0\), define the sparse cone
\begin{equation}
\label{eq:sparse-cone}
\Gamma_H(c,S)
=
\Bigl\{\Delta\in\mathbb R^p:
\|\Delta_{S^c}\|_1\le c\|\Delta_S\|_1
\Bigr\}.
\end{equation}
This cone contains error vectors whose off-support mass is controlled by their on-support mass. Restricted strong convexity over such a cone is the high-dimensional substitute for global strong convexity.

For online quantile regression, local curvature comes from the conditional density of the noise near the target quantile and from the restricted eigenvalue behavior of the design. Under the assumptions stated in Section~\ref{sec:main-theory}, the sliding-window quantile loss satisfies a restricted curvature inequality along local sparse directions. In particular, for \(\bar s\)-sparse directions \(\Delta\) with \(\|\Delta\|_2\) bounded, the appendix states a high-probability bound of the form
\[
\bigl\langle
\bar g_t(\beta^\ast+\Delta)-\bar g_t(\beta^\ast),\Delta
\bigr\rangle
\ge
a^\ast\|\Delta\|_2^2
-
c^\ast\sqrt{\frac{\bar s\log(pT/\delta)}{W}}\|\Delta\|_2.
\]
Equivalently, after absorbing the linear tolerance,
\[
\bigl\langle
\bar g_t(\beta^\ast+\Delta)-\bar g_t(\beta^\ast),\Delta
\bigr\rangle
\ge
\mu\|\Delta\|_2^2
-
C_{\mathrm{rsc}}\frac{\bar s\log(pT/\delta)}{W}.
\]
A formal statement is given in Appendix Lemma~\ref{app:lem:rsc}.

On the cone \(\Gamma_H(c,S)\), we have \(\|\Delta\|_1\le (1+c)\sqrt{s_0}\|\Delta\|_2\). Thus, when the sliding window is sufficiently informative, the quadratic term dominates the stochastic tolerance and the loss behaves locally like a strongly convex objective along sparse directions. This local curvature is the source of the fast refinement rate and the logarithmic sliding-window regret established later.

The remaining difficulty is algorithmic. AIHT performs several unprojected updates between two hard-thresholding times. If these updates are allowed to move all \(p\) coordinates, then coordinatewise stochastic noise can accumulate in \(\ell_1\) over many irrelevant coordinates. To avoid this dimension-dependent leakage, AIHT uses a candidate-restricted update direction. At each epoch, the algorithm updates only a controlled working set consisting of the current active coordinates and a small number of candidate coordinates selected from the largest subgradient components. This screening step is not a sparse projection; it only restricts the direction of the unprojected updates. The hard-thresholding projection is still applied only at epoch endpoints.
    
\subsection{Gradient leakage geometry}
\label{subsec:gradient-leakage}

Restricted curvature is useful only while the error vector remains inside a sparse cone. A hard-thresholding step promotes sparsity at the end of an epoch, but within an epoch AIHT uses ordinary subgradient movement on a candidate set. These updates may create off-support coordinates, so we need to control how much off-support movement can accumulate before the next projection.

We decompose vectors relative to the true support \(S=\operatorname{supp}(\beta^\ast)\). For any \(v\in\mathbb R^p\), \(v_S\) denotes the coordinates on \(S\), and \(v_{S^c}\) denotes the coordinates outside \(S\). When the current error \(\Delta=\beta-\beta^\ast\) lies in the sparse cone, the off-support part of the quantile subgradient is controlled coordinatewise by two terms: a deterministic leakage term proportional to the on-support error, and a stochastic score fluctuation determined by the sliding-window length.

The following lemma gives the basic coordinatewise leakage bound.

\begin{lemma}[Gradient leakage for quantile loss]
\label{lem:grad-leak-quantile}
Suppose \textnormal{(A1)--(A4)} hold and let \(t\ge W\). Fix any set \(J\subseteq\{1,\ldots,p\}\) such that \(S\subseteq J\) and \(|J|\le \bar s=s+m+s_0\). For any \(\beta\) with \(\Delta:=\beta-\beta^\ast\), \(\operatorname{supp}(\Delta)\subseteq J\), \(\Delta\in\Gamma(c,S)\), and \(\|\Delta\|_2\le R\), there exist constants \(\rho>0\) and \(C_0>0\) such that, for any \(\delta\in(0,1)\) and any \(W\ge C\log(p/\delta)\), with probability at least \(1-\delta\),
\[
\|\bar g_{t,S^c}(\beta)\|_\infty
\le
\rho\|\Delta_S\|_1
+
C_0\sqrt{\frac{\log(p/\delta)}{W}}.
\]
\end{lemma}

The deterministic term \(\rho\|\Delta_S\|_1\) reflects covariance leakage from active coordinates into inactive coordinates. The stochastic term \(\sigma_W\) is the coordinatewise score fluctuation induced by the finite sliding window. Thus, once the iterate is aligned with the sparse cone, no single off-support coordinate receives a large update unless the on-support error is still large or the window is too noisy.

However, cone stability is an \(\ell_1\) statement, not an \(\ell_\infty\) statement. A direct full off-support \(\ell_1\) bound would generally scale with \(p\), because it would sum stochastic fluctuations over all inactive coordinates. The candidate-restricted update avoids this problem by allowing only a controlled number of off-support coordinates to move within an epoch.

Let \(A_j\) denote the candidate set used during epoch \(j\), and define the update direction
\[
d_t
=
P_{A_j}\bar g_t(\beta_t),
\qquad t\in\mathcal T_j,
\]
where coordinates outside \(A_j\) are set to zero. The set \(A_j\) is chosen so that \(|A_j|\le s+m\), where \(s\) is the working sparsity and \(m\) is the candidate size. Since \(m\) is chosen on the order of the sparsity rather than the ambient dimension, the resulting leakage bounds are sparsity-dependent but not dimension-dependent.

\begin{lemma}[Candidate-restricted \(\ell_1\) leakage]
\label{lem:candidate-l1-leakage}
Suppose the coordinatewise leakage bound in Lemma~\ref{lem:grad-leak-quantile} and the corresponding on-support coordinatewise bound hold on an event \(\mathcal E_\infty\). Let \(A_j\) be the candidate set used during epoch \(j\), with \(|A_j|\le s+m\), and let \(d_t=P_{A_j}\bar g_t(\beta_t)\). Then, on \(\mathcal E_\infty\), for all \(t\in\mathcal T_j\),
\begin{equation}
\label{eq:candidate-l1-off}
\|d_{t,S^c}\|_1
\le
(s+m)\rho\|\Delta_{t,S}\|_1
+
(s+m)\sigma_W,
\end{equation}
and
\begin{equation}
\label{eq:candidate-l1-on}
\|d_{t,S}\|_1
\le
s_0L_g\|\Delta_{t,S}\|_1
+
s_0\sigma_W,
\end{equation}
where \(L_g\) is the constant in the on-support coordinatewise gradient bound. Equivalently, with \(\zeta_W=\sigma_W\), these inequalities imply inflated-cone guards of the form
\[
\|d_{t,S^c}\|_1
\le
L_{\mathrm{off}}\{\|\Delta_{t,S}\|_1+\zeta_W\},
\qquad
\|d_{t,S}\|_1
\le
L_{\mathrm{on}}\{\|\Delta_{t,S}\|_1+\zeta_W\},
\]
where \(L_{\mathrm{off}}\) depends on \(s+m\) and \(\rho\), and \(L_{\mathrm{on}}\) depends on \(s_0\) and \(L_g\), but neither constant depends on \(p\) except through the logarithmic factor inside \(\zeta_W\).
\end{lemma}

The proof is immediate from Lemma~\ref{lem:grad-leak-quantile}: the off-support part of the direction \(d_t\) has at most \(s+m\) nonzero coordinates, so its \(\ell_1\) norm is bounded by \((s+m)\|\bar g_{t,S^c}(\beta_t)\|_\infty\). The on-support bound is analogous, using \(|S|=s_0\).

This lemma is the bridge between the coordinatewise quantile concentration result and the \(\ell_1\)-based cone-stability argument. It shows why the candidate restriction is important: without it, the stochastic off-support \(\ell_1\) term would scale with \(p\); with it, the leakage budget scales with the working sparsity and candidate size.

This gives the following guardrail geometry. In the post-burn-in regime, hard thresholding returns the iterate to a narrow sparse region. During the next epoch, candidate-restricted subgradient steps may introduce controlled off-support leakage into a larger cone. If the epoch mass \(B_j=\sum_{t\in\mathcal T_j}\eta_t\) is not too large, the iterate remains inside this larger cone until the next thresholding step.
\subsection{Adaptive Iterative Hard Thresholding}
\label{subsec:aiht}

We now present the AIHT algorithm. The update has two layers. The first layer is a stochastic subgradient step restricted to a candidate set. The second layer is an intermittent hard-thresholding step applied only at the end of each epoch. The candidate restriction controls dense off-support noise, while delayed thresholding allows potentially relevant coordinates to accumulate signal before being judged by the sparse projection.

At the beginning of epoch \(j\), let \(\widehat S_j=\operatorname{supp}(\beta_{\tau_j})\). The algorithm forms an exploration set \(C_j\) by selecting the \(m\) largest coordinates of \(|\bar g_{\tau_j}(\beta_{\tau_j})|\) outside \(\widehat S_j\), and then sets
\[
A_j=\widehat S_j\cup C_j.
\]
Thus \(|A_j|\le s+m\). During epoch \(j\), only coordinates in \(A_j\) are updated. At the end of the epoch, \(H_s\) is applied to return the iterate to the working sparsity level \(s\).

\begin{algorithm}[htbp]
    \caption{Adaptive Iterative Hard Thresholding (AIHT)}
    \label{alg:aiht-adaptive}
    \begin{algorithmic}[1]

        \State \textbf{Input:} working sparsity \(s\ge s_0\), candidate size \(m\), initial epoch length \(k_1\), window size \(W\), decay factor \(\gamma\in(0,1)\), threshold \(\tau_{\mathrm{lo}}>0\), hysteresis width \(r\), stabilization constant \(\epsilon_0>0\), step-size parameters \((\alpha_1,b_1)\) for Phase~I and \((\alpha_2,b_2)\) for Phase~II, minimum period \(k_{\min}\), and optional epoch-mass cap \(B_{\max}\).
        \State Initialize \(\beta_1=0\), \(\mathsf{PHASE}=\mathrm{I}\), \(j\gets1\), \(\tau_j\gets1\), and \(k_j\gets k_1\).

        \For{\(t=1,2,\ldots\)}
            \State Compute the sliding-window subgradient \(\bar g_t(\beta_t)\) using \eqref{move_grad}.

            \If{\(t=\tau_j\)}
                \State Set \(\widehat S_j=\operatorname{supp}(\beta_{\tau_j})\).
                \State Let \(C_j\) be the indices of the \(m\) largest coordinates of \(|\bar g_{\tau_j}(\beta_{\tau_j})|\) outside \(\widehat S_j\).
                \State Set \(A_j=\widehat S_j\cup C_j\).
            \EndIf

            \State Choose the step size
            \[
            \eta_t=
            \begin{cases}
                \alpha_1/\sqrt{t+b_1}, & \mathsf{PHASE}=\mathrm{I},\\[3pt]
                \alpha_2/(t+b_2), & \mathsf{PHASE}=\mathrm{II}.
            \end{cases}
            \]

            \State Restrict the update direction to the candidate set:
            \[
            d_t=P_{A_j}\bar g_t(\beta_t).
            \]

            \State Take one candidate-restricted subgradient step:
            \[
            \widetilde\beta_{t+1}=\beta_t-\eta_t d_t.
            \]

            \State Calculate the projected-gradient mapping statistic using the same direction \(d_t\):
            \[
            G_{\eta_t}^{A_j}(\beta_t)
            =
            \frac{1}{\eta_t}
            \Bigl(
            \beta_t-H_s(\beta_t-\eta_t d_t)
            \Bigr).
            \]

            \If{\(t=\tau_j+k_j-1\)}
                \State Apply hard thresholding:
                \[
                \beta_{t+1}=H_s(\widetilde\beta_{t+1}).
                \]
                \State Calculate \(\bar G_j\), \(M_j\), and \(\mathrm{Var}_j\) using \eqref{Gbar}--\eqref{SNR}, with \(G_{\eta_t}^{A_j}\) in place of \(G_{\eta_t}\).
                \State Determine whether to switch phase using the transition criterion \eqref{criteria}.

               \State Set the provisional next epoch length
                \[
                    k_{j+1}^{\mathrm{prop}}
                    =
                    \begin{cases}
                    k_j, & \mathsf{PHASE}=\mathrm{I},\\[2pt]
                    \max\{k_{\min},\lfloor\gamma k_j\rfloor\}, & \mathsf{PHASE}=\mathrm{II}.
                    \end{cases}
                \]
                \State If an epoch-mass cap \(B_{\max}\) is used, set \(k_{j+1}\) to the largest integer \(k\le k_{j+1}^{\mathrm{prop}}\) such that
                \[
                    \sum_{u=t+1}^{t+k}\eta_u\le B_{\max}.
                \]
                Otherwise set \(k_{j+1}=k_{j+1}^{\mathrm{prop}}\).
\State Set \(j\gets j+1\), and \(\tau_j\gets t+1\).
            \Else
                \State Set \(\beta_{t+1}=\widetilde\beta_{t+1}\).
            \EndIf
        \EndFor
    \end{algorithmic}
\end{algorithm}

The candidate set \(A_j\) plays a different role from the hard-thresholding operator. Candidate screening controls the dimension of the update direction inside an epoch, preventing stochastic noise from accumulating over all \(p\) off-support coordinates. Hard thresholding, in contrast, controls the sparsity of the iterate at epoch endpoints. A missing true coordinate can enter the model when its subgradient signal is large enough to be included in \(C_j\) and then accumulates sufficient mass before the next projection.

The phase schedule follows the support-discovery and refinement logic described above. In Phase~I, AIHT uses a larger step size and relatively delayed thresholding, giving candidate coordinates time to grow. In Phase~II, the step size decreases at the \(1/t\) rate and the projection frequency increases, which stabilizes the sparse estimator and allows the method to exploit restricted local curvature.

The transition from Phase~I to Phase~II is monitored through a projected-gradient mapping statistic. For epoch \(j\), define
\begin{equation}
\label{Gbar}
\bar G_j
=
\frac{1}{k_j}\sum_{u=\tau_j}^{\tau_j+k_j-1}
G_{\eta_u}^{A_j}(\beta_u),
\qquad
M_j^2=\|\bar G_j\|_2^2,
\end{equation}
and
\begin{equation}
\label{variance}
\mathrm{Var}_j
=
\frac{1}{k_j}\sum_{u=\tau_j}^{\tau_j+k_j-1}
\bigl\|G_{\eta_u}^{A_j}(\beta_u)-\bar G_j\bigr\|_2^2.
\end{equation}
The ratio
\begin{equation}
\label{SNR}
R_j
=
\frac{M_j^2}{\mathrm{Var}_j+\epsilon_0}
\end{equation}
acts as a stability diagnostic. AIHT switches to Phase~II only after this ratio remains below the threshold \(\tau_{\mathrm{lo}}\) for \(r\) consecutive epochs:
\begin{equation}
\label{criteria}
R_{j-i}\le \tau_{\mathrm{lo}},
\qquad
i=0,\ldots,r-1.
\end{equation}

The epoch-mass cap \(B_{\max}\) is included to align the implementation with the cone-stability mechanism from Proposition~\ref{prop:epoch-mass-cone-barrier}. The theory in Section~\ref{sec:main-theory} states the required upper bound in terms of \(B_j=\sum_{t\in\mathcal T_j}\eta_t\). In practice, the gradient-mapping diagnostic decides when to enter the refinement phase, while the shrinking epoch length and optional mass cap keep the amount of unprojected movement within a stable range.

\section{Main Theory}
	\label{sec:main-theory}
	In this section, we study AIHT under the stationary data-generating model introduced in Section~\ref{intro}. Our goal is to show that the sliding-window quantile loss possesses sufficient local curvature, that the iterates remain inside a sparse guardrail cone under intermittent thresholding, and that these two ingredients yield both a two-phase convergence pattern and logarithmic regret. The gradient-leakage inequality was established in Section~\ref{sec:methodology} as a methodological preliminary, so we do not repeat it here. Full proofs of all results in this section are deferred to the Appendix.
\subsection{Assumptions and local curvature}
\label{subsec:assumptions-curvature}

We begin with the standing assumptions for the stationary quantile-regression model. These assumptions separate the statistical regularity of the data from the post-burn-in algorithmic conditions used later in the cone-invariance and convergence theorems. Throughout this section, \(S=\operatorname{supp}(\beta^\ast)\), \(|S|=s_0\), and the algorithm uses working sparsity \(s\ge s_0\) and candidate size \(m\). We write
\[
\bar s=s+m+s_0
\]
for a generic upper bound on the size of the sparse coordinate sets that arise from the union of the true support and the candidate-restricted update set.

\begin{description}

    \item[(A1)] The covariate vectors \(X_t\in\mathbb R^p\) are i.i.d. mean-zero sub-Gaussian random vectors. There exists \(K<\infty\) such that
    \[
    \sup_{\|u\|_2=1}\|u^\top X_t\|_{\psi_2}\le K,
    \qquad
    \Sigma=\mathbb E[X_tX_t^\top].
    \]

    \item[(A2)] The conditional \(\tau\)-quantile of the noise is zero:
    \[
    \mathbb P(\varepsilon_t\le 0\mid X_t)=\tau .
    \]
    Moreover, the conditional distribution of \(\varepsilon_t\) given \(X_t\) admits a density \(f_{\varepsilon|X}(\cdot\mid X_t)\) in a neighborhood of the origin. There exist \(b_0>0\), \(0<m_0\le M_0<\infty\), and \(M_1<\infty\) such that \(f_{\varepsilon|X}(u\mid X_t)\ge m_0\) for \(u\in[-b_0,b_0]\), and \(f_{\varepsilon|X}(u\mid X_t)\le M_1\) for all \(u\in\mathbb R\).

    \item[(A3)] The population covariance has restricted curvature on sparse directions. There exists \(m_1>0\) such that
    \begin{equation}
    \label{eq:sparse-re-main}
    \Delta^\top\Sigma\Delta
    \ge
    m_1\|\Delta\|_2^2
    \qquad
    \text{for all } \Delta\in\mathbb R^p
    \text{ with } \|\Delta\|_0\le \bar s .
    \end{equation}

    \item[(A4)] Let
    \[
    \widehat\Sigma_t
    =
    \frac1W\sum_{i=t-W+1}^{t}X_iX_i^\top,
    \qquad t\ge W .
    \]
    Write \(\|\cdot\|_{\infty\to\infty}\) for the maximum absolute row-sum norm. There exists \(C_\infty<\infty\) such that, on the high-probability event used in the analysis,
    \begin{equation}
    \label{eq:sparse-gram-main}
    \sup_{t\ge W}
    \sup_{\substack{A\subseteq\{1,\ldots,p\}\\ S\subseteq A,\ |A|\le \bar s}}
    \left\{
    \|\widehat\Sigma_{t,S,A}\|_{\infty\to\infty}
    +
    \|\widehat\Sigma_{t,S^c,A}\|_{\infty\to\infty}
    \right\}
    \le
    C_\infty .
    \end{equation}

\end{description}

Assumption (A1) is a standard high-dimensional design condition. It gives the concentration needed for sliding-window score bounds and sparse empirical curvature. Assumption (A2) identifies \(\beta^\ast\) as the population \(\tau\)-quantile parameter and supplies the local density lower bound that produces curvature for the check loss. No moment condition on \(\varepsilon_t\) is imposed, which is one reason quantile regression remains suitable under heavy-tailed noise.

Assumption (A3) is the sparse analogue of a restricted-eigenvalue condition. Since the revised AIHT update is restricted to candidate sets of size at most \(s+m\), the relevant post-burn-in error directions are supported on the union of the true support and a candidate set, whose size is at most \(\bar s=s+m+s_0\). Thus the curvature condition is required only on sparse directions, not on all of \(\mathbb R^p\).

Assumption (A4) is the empirical leakage-control condition corresponding to the candidate-restricted update. It is intentionally stated only over small sets \(A\) with \(|A|\le\bar s\), rather than over the full \(S^c\times S^c\) block of the empirical Gram matrix. This distinction is important: a full off-support row-sum bound can scale with the ambient dimension \(p\), while the candidate-restricted version controls only the coordinates that can actually be updated within an epoch. Under sub-Gaussian designs with bounded sparse covariance interactions, such bounds hold with high probability when the window length is large enough relative to \(\bar s\log p\).

Under (A1)--(A3), the sliding-window quantile loss has restricted local curvature along the sparse directions used by AIHT. In particular, on a high-probability event, for local \(\bar s\)-sparse directions \(\Delta\),
\[
\bigl\langle
\bar g_t(\beta^\ast+\Delta)-\bar g_t(\beta^\ast),\Delta
\bigr\rangle
\ge
a^\ast\|\Delta\|_2^2
-
c^\ast\sqrt{\frac{\bar s\log p}{W}}\|\Delta\|_2,
\]
or equivalently, after absorbing the linear tolerance,
\[
\bigl\langle
\bar g_t(\beta^\ast+\Delta)-\bar g_t(\beta^\ast),\Delta
\bigr\rangle
\ge
\mu\|\Delta\|_2^2
-
C_{\mathrm{rsc}}\frac{\bar s\log p}{W}.
\]
The appendix gives the formal high-probability statement. This local curvature is the statistical ingredient used later; the algorithmic ingredient is the candidate-restricted leakage control and epoch-mass condition, which keep the iterates inside the region where this curvature is valid.

For the local results below, we work in the post-burn-in regime. Fix constants
\(0<c_{\mathrm{small}}<c_{\mathrm{large}}\) and \(0<R_{\mathrm{in}}<R\). Let
\[
\zeta_W
=
C_\zeta\sqrt{\frac{\log(pT/\delta)}{W}},
\]
where \(C_\zeta\) is chosen large enough to dominate the coordinatewise score fluctuations appearing in the gradient-leakage bounds. Define the inflated sparse cone
\[
\Gamma_\zeta(c,S)
=
\Bigl\{
\Delta\in\mathbb R^p:
\|\Delta_{S^c}\|_1
\le
c\{\|\Delta_S\|_1+\zeta_W\}
\Bigr\}.
\]

During epoch \(j\), AIHT uses a candidate set \(A_j\) with \(|A_j|\le s+m\), and the actual update direction is
\[
d_t=P_{A_j}\bar g_t(\beta_t),
\qquad t\in\mathcal T_j,
\]
where \(P_Av\) denotes the vector equal to \(v\) on \(A\) and zero outside \(A\). On the high-probability gradient-control event used below, whenever
\(\Delta_t\in\Gamma_\zeta(c_{\mathrm{large}},S)\) and \(\|\Delta_t\|_2\le R\), the candidate-restricted direction satisfies
\[
\|d_{t,S^c}\|_1
\le
L_{\mathrm{off}}\{\|\Delta_{t,S}\|_1+\zeta_W\},
\qquad
\|d_{t,S}\|_1
\le
L_{\mathrm{on}}\{\|\Delta_{t,S}\|_1+\zeta_W\}.
\]
Here \(L_{\mathrm{off}}\) depends on the candidate size \(s+m\) and the coordinatewise leakage constant, while \(L_{\mathrm{on}}\) depends on \(s_0\) and the on-support coordinatewise gradient bound. These constants do not depend on the ambient dimension \(p\), except through the logarithmic factor already present in \(\zeta_W\).

After some burn-in epoch \(J_0\), the post-thresholding iterates are assumed to satisfy
\[
\Delta_{\tau_j}\in\Gamma_\zeta(c_{\mathrm{small}},S),
\qquad
\|\Delta_{\tau_j}\|_2\le R_{\mathrm{in}},
\qquad
j\ge J_0.
\]
The epoch schedule is chosen so that, for all \(j\ge J_0\),
\[
\eta_tL_{\mathrm{on}}\le \frac12,
\qquad
B_j:=\sum_{t\in\mathcal T_j}\eta_t\le B_\star,
\]
where
\[
B_\star
\le
\min\left\{
\frac{c_{\mathrm{large}}-c_{\mathrm{small}}}
{2\{L_{\mathrm{off}}+c_{\mathrm{large}}L_{\mathrm{on}}\}},
\,
\frac{R-R_{\mathrm{in}}}
{(L_{\mathrm{off}}+L_{\mathrm{on}})(\sqrt{s_0}R+\zeta_W)}
\right\}.
\]
The first upper bound is the cone-stability budget, and the second keeps the within-epoch trajectory inside the local radius \(R\) where the gradient-leakage and curvature bounds are valid.

For the convergence and regret results, we use the following local post-burn-in setup. After the burn-in epoch \(J_0\), the candidate sets are assumed to cover the true support:
\[
S\subseteq A_j,
\qquad j\ge J_0.
\]
Since the iterate is supported on \(A_j\) within epoch \(j\), this implies that \(\Delta_t=\beta_t-\beta^\ast\) is supported on \(A_j\) for all \(t\in\mathcal T_j\). Consequently,
\[
\langle d_t,\Delta_t\rangle
=
\langle \bar g_t(\beta_t),\Delta_t\rangle,
\qquad
d_t=P_{A_j}\bar g_t(\beta_t).
\]

At post-burn-in thresholding times, we also work in the support-preserving regime:
\[
S\subseteq \operatorname{supp}\{H_s(\widetilde\beta_{\tau_{j+1}})\},
\qquad j\ge J_0.
\]
This condition guarantees that hard thresholding is nonexpansive relative to \(\beta^\ast\):
\[
\|H_s(v)-\beta^\ast\|_2\le \|v-\beta^\ast\|_2
\quad
\text{whenever } S\subseteq\operatorname{supp}\{H_s(v)\}.
\]
A sufficient beta-min condition for this support preservation is stated after Theorem~\ref{thm:two-phase-convergence}.

Finally, on the local high-probability event used below, the sliding-window quantile loss satisfies the one-point restricted curvature inequalities along the AIHT path:
\[
\bigl\langle
\bar g_t(\beta_t)-\bar g_t(\beta^\ast),\Delta_t
\bigr\rangle
\ge
\mu\|\Delta_t\|_2^2,
\]
and
\[
Q_t(\beta_t)-Q_t(\beta^\ast)
\le
\bigl\langle \bar g_t(\beta_t),\Delta_t\bigr\rangle
-
\frac{\mu}{2}\|\Delta_t\|_2^2.
\]
The score-at-truth term is controlled in the stochastic approximation sense:
\[
-2\eta_t
\bigl\langle
\bar g_t(\beta^\ast),\Delta_t
\bigr\rangle
\le
\frac{\mu}{2}\eta_t\|\Delta_t\|_2^2
+
C_{\mathrm{sc}}\eta_t^2.
\]
We also assume that the candidate-restricted directions are bounded on this event:
\[
\|d_t\|_2\le G_{\bar s},
\qquad
\bar s=s+m+s_0.
\]
The constants \(\mu,C_{\mathrm{sc}},G_{\bar s}\) may depend on the local curvature, the window size, and the sparse/candidate dimension \(\bar s\), but they do not scale polynomially with the ambient dimension \(p\). The dependence on \(p\) enters through logarithmic concentration factors.
	
\subsection{Gradient control and cone invariance}
\label{subsec:gradient-control-cone}

The restricted curvature bound applies only inside a sparse local region. We therefore need to show that AIHT does not leave this region between two hard-thresholding steps. The key point is that the update direction is candidate-restricted. Instead of updating all \(p\) coordinates, AIHT updates only
\[
d_t=P_{A_j}\bar g_t(\beta_t),
\qquad t\in\mathcal T_j,
\]
where \(|A_j|\le s+m\). This restriction converts the coordinatewise leakage bound from Section~\ref{subsec:gradient-leakage} into an \(\ell_1\) leakage bound whose constants depend on the sparsity and candidate size, not on the ambient dimension.

Indeed, on the gradient-control event described in Section~\ref{subsec:assumptions-curvature}, the coordinatewise leakage bounds imply
\[
\|d_{t,S^c}\|_1
\le
L_{\mathrm{off}}\{\|\Delta_{t,S}\|_1+\zeta_W\},
\qquad
\|d_{t,S}\|_1
\le
L_{\mathrm{on}}\{\|\Delta_{t,S}\|_1+\zeta_W\}.
\]
The inflated tolerance \(\zeta_W\) is necessary near the truth: when \(\|\Delta_{t,S}\|_1\) is small, finite-window score fluctuations cannot be absorbed into a pure cone ratio. The inflated cone \(\Gamma_\zeta(c,S)\) keeps the cone condition meaningful throughout local refinement.

We now state the cone-invariance result. The theorem controls the pre-threshold trajectory within each epoch. At the end of the epoch, hard thresholding produces the next post-thresholding iterate, which is covered by the post-burn-in condition stated in Section~\ref{subsec:assumptions-curvature}.

\begin{theorem}[Cone invariance for candidate-restricted AIHT]
\label{thm:cone-invariance}
Under {\normalfont (A1)}--{\normalfont (A4)} and the post-burn-in setup stated above, suppose the gradient-control event holds. Then for every epoch \(j\ge J_0\), the within-epoch pre-threshold trajectory remains in the enlarged inflated cone:
\[
\Delta_t\in\Gamma_\zeta(c_{\mathrm{large}},S),
\qquad
\|\Delta_t\|_2\le R,
\qquad
t\in\mathcal T_j.
\]
Moreover, the provisional endpoint before the epoch-end thresholding step also satisfies
\[
\widetilde\Delta_{\tau_{j+1}}
:=
\widetilde\beta_{\tau_{j+1}}-\beta^\ast
\in
\Gamma_\zeta(c_{\mathrm{large}},S),
\qquad
\|\widetilde\Delta_{\tau_{j+1}}\|_2\le R.
\]
\end{theorem}

\paragraph{Remarks.}
The theorem formalizes the guardrail mechanism of AIHT. Each epoch begins, after thresholding, inside a smaller inflated cone. Candidate-restricted subgradient updates may introduce off-support leakage, but the leakage is controlled in \(\ell_1\) because only \(s+m\) coordinates are updated. If the epoch mass \(B_j=\sum_{t\in\mathcal T_j}\eta_t\) is below the cone-stability budget, the iterate cannot leave the larger cone before the next thresholding step.

The constants in the cone-stability budget depend on the working sparsity \(s\), the candidate size \(m\), and the true sparsity \(s_0\), but not polynomially on \(p\). The ambient dimension enters through the logarithmic factor in \(\zeta_W\), which is the usual price for high-dimensional score concentration.

\subsection{Convergence analysis}
\label{subsec:convergence}

Once cone invariance is established, the sliding-window loss provides usable local curvature along the candidate-restricted AIHT trajectory. The resulting behavior has two components. During the discovery stage, epoch-level contraction depends on the accumulated step-size mass. After the phase transition, the algorithm enters a refinement regime in which the step size has order \(1/t\), thresholding is more frequent, and the support structure is stable.

Let \(T_0\) denote the first iteration after which the phase-transition criterion in \eqref{criteria} is triggered, so that AIHT enters Phase~II at time \(T_0+1\). The theorem below does not attempt to prove an optimal hitting time for this statistic. Instead, it describes the behavior of the iterates once the post-burn-in local conditions stated in Section~\ref{subsec:assumptions-curvature} hold.

\begin{theorem}[Two-phase convergence of candidate-restricted AIHT]
\label{thm:two-phase-convergence}
Under {\normalfont (A1)}--{\normalfont (A4)} and the post-burn-in local setup of Section~\ref{subsec:assumptions-curvature}, suppose the cone-invariance conclusion of Theorem~\ref{thm:cone-invariance} holds for all post-burn-in epochs. Then there exist constants \(C>0\) and \(\mu>0\) such that the following statements hold on the local high-probability event.

\medskip
\noindent\textbf{Phase I.}
For every post-burn-in epoch \(j\) before the transition time \(T_0\), the post-thresholding iterates satisfy
\begin{equation}
\label{eq:phaseI-epoch-recursion}
\|\Delta_{\tau_{j+1}}\|_2^2
\le
a_j\|\Delta_{\tau_j}\|_2^2
+
C\sum_{t\in\mathcal T_j}\eta_t^2,
\qquad
a_j
\le
\exp\!\left(-\mu\sum_{t\in\mathcal T_j}\eta_t\right).
\end{equation}
In particular, if the Phase~I schedule keeps the epoch mass bounded below, \(B_j=\sum_{t\in\mathcal T_j}\eta_t\ge B_{\min}>0\), then the post-thresholding errors contract geometrically up to a stochastic tolerance:
\[
\|\Delta_{\tau_{j+1}}\|_2
\le
\lambda\|\Delta_{\tau_j}\|_2
+
C\left(\sum_{t\in\mathcal T_j}\eta_t^2\right)^{1/2},
\qquad
\lambda=\exp(-\mu B_{\min}/2)<1.
\]

\medskip
\noindent\textbf{Phase II.}
After the transition time \(T_0\), suppose the step size is
\[
\eta_t=\frac{1}{c_{\mathrm{eff}}(t+b_2)},
\qquad
0<c_{\mathrm{eff}}\le \mu,
\]
with \(b_2\) large enough that the epoch-mass and local-radius conditions remain satisfied. Then, for all \(t\ge T_0+1\),
\begin{equation}
\label{eq:phaseII-rate}
\|\Delta_t\|_2^2
\le
C\,\frac{\log(t+b_2)}{t+b_2}.
\end{equation}
\end{theorem}

\paragraph{Remarks.}
The Phase~I statement is expressed in terms of epoch mass. This avoids the misleading implication that a fixed number of iterations always yields the same contraction: under a decaying step size, the contraction factor is controlled by \(\sum_{t\in\mathcal T_j}\eta_t\), not by \(k_j\) alone.

The Phase~II result describes the stable local refinement regime. Once the candidate sets cover the true support and thresholding preserves it, the candidate-restricted direction has the same inner product with the error vector as the full sliding-window subgradient. Thus AIHT behaves locally like a sparse stochastic approximation method with restricted curvature. The decreasing step size and more frequent thresholding then yield the \(O((\log t)/t)\) error rate.

The support-preserving condition is natural after sufficiently accurate recovery. Let \(\beta_{\min}=\min_{j\in S}|\beta_j^\ast|\). If at a thresholding time
\[
\|\widetilde\beta_{\tau_{j+1}}-\beta^\ast\|_2<\frac{\beta_{\min}}{2},
\]
then every true coordinate belongs to \(\operatorname{supp}\{H_s(\widetilde\beta_{\tau_{j+1}})\}\) whenever \(s\ge s_0\). Thus, after the estimator enters a beta-min neighborhood, true-support containment becomes self-sustaining. In a nonstationary setting, this containment may fail after a changepoint; Section~\ref{sec:piecewise} introduces a restart rule that returns the algorithm to the discovery phase when the local score geometry changes.

\subsection{Regret analysis}
\label{subsec:regret}

We now state the online optimization guarantee for the sliding-window objective optimized by AIHT. To avoid confusion with the phase-transition statistic \(R_j\), we denote cumulative regret by \(\mathcal R_T\):
\[
\mathcal R_T
=
\sum_{t=1}^T
\bigl\{Q_t(\beta_t)-Q_t(\beta^\ast)\bigr\}.
\]
The theorem separates the discovery cost from the refinement-phase regret. This is important because the transition time \(T_0\) is determined by the stabilization diagnostic and is not analyzed here as an optimal stopping time.

\begin{theorem}[Sliding-window regret of candidate-restricted AIHT]
\label{thm:regret-aiht}
Under the conditions of Theorem~\ref{thm:two-phase-convergence}, suppose Phase~II uses
\[
\eta_t=\frac{1}{c_{\mathrm{eff}}(t+b_2)},
\qquad
0<c_{\mathrm{eff}}\le \mu.
\]
Then, on the same local high-probability event, for every \(T\ge T_0+1\),
\begin{equation}
\label{eq:regret-main}
\mathcal R_T
\le
C_{\mathrm I}(T_0)
+
C_{\mathrm{burn}}
+
\frac{G_{\bar s}^2}{2c_{\mathrm{eff}}}
\log\!\left(\frac{T+b_2}{T_0+1+b_2}\right),
\end{equation}
where \(C_{\mathrm{burn}}\) is the finite pre-localization cost and
\(C_{\mathrm I}(T_0)=O(\sqrt{T_0})\) is the discovery-phase contribution under the Phase~I step size \(\eta_t\asymp t^{-1/2}\). In particular, the refinement-phase regret is logarithmic in \(T\). If \(T_0\) and the burn-in cost are bounded independently of \(T\), then \(\mathcal R_T=O(\log T)\).
\end{theorem}

\paragraph{Remarks.}
The logarithmic term is the same time dependence as in strongly convex online optimization, but here it is obtained only after the trajectory has entered the sparse local region. The constant \(G_{\bar s}\) is the bound on the candidate-restricted update direction, so it depends on the sparse/candidate dimension \(\bar s=s+m+s_0\), rather than on the full ambient dimension \(p\).

The regret is stated for the sliding-window objective \(Q_t\), which is exactly the objective used to compute the AIHT update. Extending the bound to instantaneous losses \(\ell_t\) would require an additional comparison argument and is not needed for the main mechanism studied here.

\section{AIHT under Distributional Shift}
\label{sec:piecewise}

The stationary theory in Section~\ref{sec:main-theory} assumes that the sparse target parameter is fixed. We now consider a piecewise-stationary stream in which the target may change at unknown times. A changepoint can invalidate the support coverage, cone stability, and local refinement conditions used in Section~\ref{sec:main-theory}. Thus, after a structural change, the algorithm should return to the support-discovery regime rather than continuing with the old candidate set and the old Phase~II step size.

We handle distributional shift by adding a restart rule to candidate-restricted AIHT. When a change is detected, the algorithm performs a hard restart: it sets the next iterate to zero, flushes the local buffer, resets the local clock and epoch counter, returns to Phase~I, and clears the candidate set and phase-transition diagnostic history. This makes the post-change run a new local instance of AIHT. After enough clean observations from the new segment have accumulated, the local post-burn-in theory from Section~\ref{sec:main-theory} can be applied with the new target and support.

We use an adjacent-window score contrast as the change detector. Let \(h\) be a detection-window length. For \(t\ge 2h\), define
\[
\mathcal I_t^-=\{t-2h+1,\ldots,t-h\},
\qquad
\mathcal I_t^+=\{t-h+1,\ldots,t\}.
\]
For any vector \(\beta\), define the two window scores
\[
\bar g_t^-(\beta)
=
-\frac1h\sum_{i\in\mathcal I_t^-}
X_i\{\tau-\mathbf 1(Y_i\le X_i^\top\beta)\},
\qquad
\bar g_t^+(\beta)
=
-\frac1h\sum_{i\in\mathcal I_t^+}
X_i\{\tau-\mathbf 1(Y_i\le X_i^\top\beta)\}.
\]
The detector evaluates both scores at the frozen reference point \(\breve\beta_t=\beta_{t-2h}\). Since \(\breve\beta_t\) is computed before both detection windows, the detector avoids using the same observations both to construct the reference point and to test for a change. The score contrast is
\begin{equation}
\label{eq:score-contrast}
D_t
=
\|\bar g_t^+(\breve\beta_t)-\bar g_t^-(\breve\beta_t)\|_\infty .
\end{equation}
For a confidence level \(\delta\in(0,1)\), let
\[
\lambda_h
=
C_D\sqrt{\frac{\log(pT/\delta)}{h}},
\qquad
b_h=2\lambda_h,
\]
where \(C_D\) is chosen large enough for the detector concentration event in Lemma~\ref{lem:score-detector-concentration}. The restart rule is
\[
D_t>b_h
\quad\Longrightarrow\quad
\text{hard restart at time }t+1.
\]
\subsection{Piecewise-stationary model and assumptions}
\label{subsec:piecewise-model}

Let
\[
1=\nu_0<\nu_1<\cdots<\nu_K<\nu_{K+1}=T+1
\]
be unknown changepoints. On segment \(k\), the target parameter is \(\theta_k\), and
\[
Y_t=X_t^\top\theta_k+\varepsilon_t,
\qquad
\mathbb P(\varepsilon_t\le0\mid X_t)=\tau,
\qquad
t\in[\nu_k,\nu_{k+1}-1].
\]
Let \(S_k=\operatorname{supp}(\theta_k)\), with \(|S_k|\le s_0\). For segment \(k\), write \(\mathbb E_k\) for expectation under its stationary distribution, and define the population score map
\[
\Psi_k(\beta)
=
\mathbb E_k\left[
-X\{\tau-\mathbf 1(Y\le X^\top\beta)\}
\right].
\]
At the segment truth, \(\Psi_k(\theta_k)=0\).

We use the following segmentwise assumptions.

\begin{description}

\item[(D1)]
On every clean segment, the stationary assumptions {\normalfont (A1)}--{\normalfont (A4)} and the post-burn-in local conditions from Section~\ref{sec:main-theory} hold with common constants after replacing \(\beta^\ast\) by \(\theta_k\) and \(S\) by \(S_k\). In particular, after a hard restart from \(\beta=0\) and after a clean recovery period, candidate-restricted AIHT satisfies the local cone, support-coverage, support-preserving thresholding, curvature, score-control, and direction-boundedness conditions used in Theorems~\ref{thm:two-phase-convergence} and~\ref{thm:regret-aiht}. The constants \(\mu\), \(c_{\mathrm{eff}}\), \(G_{\bar s}\), \(L_{\mathrm{off}}\), \(L_{\mathrm{on}}\), \(c_{\mathrm{small}}\), \(c_{\mathrm{large}}\), and the epoch-mass budget are common across segments. In addition, on the same event, the clean sliding-window excess loss of all iterates produced during the detection delay, discovery phase, and local refinement phase is bounded above by \(M_{\mathrm{loc}}\) per round.

\item[(D2)]
For each changepoint \(\nu_k\), define the score jump at the previous segment target by
\[
\mathfrak J_k
=
\|\Psi_k(\theta_{k-1})-\Psi_{k-1}(\theta_{k-1})\|_\infty
=
\|\Psi_k(\theta_{k-1})\|_\infty .
\]
There exist constants \(r_{\mathrm{det}}>0\) and \(L_{\mathrm{det}}>0\) such that both \(\Psi_{k-1}\) and \(\Psi_k\) are \(L_{\mathrm{det}}\)-Lipschitz in the ball \(B(\theta_{k-1},r_{\mathrm{det}})\):
\[
\|\Psi_\ell(\beta)-\Psi_\ell(\beta')\|_\infty
\le
L_{\mathrm{det}}\|\beta-\beta'\|_2,
\qquad
\ell\in\{k-1,k\}.
\]
The score jump is separated from the stochastic detector fluctuation:
\[
\min_{1\le k\le K}\mathfrak J_k
\ge
4\lambda_h+2L_{\mathrm{det}}r_{\mathrm{det}}.
\]

\item[(D3)]
Let \(T_{\mathrm{rec}}(r_{\mathrm{det}})\) be the number of clean observations after a hard restart needed for candidate-restricted AIHT to refill the sliding window, enter the post-burn-in local regime of Section~\ref{sec:main-theory}, and reach
\[
\|\beta_t-\theta_k\|_2\le r_{\mathrm{det}}.
\]
Each segment length satisfies
\[
\nu_{k+1}-\nu_k
\ge
T_{\mathrm{rec}}(r_{\mathrm{det}})+3h,
\qquad
k=0,\ldots,K.
\]

\end{description}

Assumption (D1) is the segmentwise analogue of the local stationary theory. It says that after a hard restart and a sufficient number of clean observations, the same candidate-restricted AIHT mechanism analyzed in Section~\ref{sec:main-theory} applies on the new segment. Assumption (D2) is a detectability condition: the expected score under the new segment, evaluated near the old target, must change by more than the detector's stochastic fluctuation. Assumption (D3) ensures that each segment is long enough for recovery before the next changepoint and leaves enough room for the frozen-reference detector.

\subsection{Restart-enabled candidate-restricted AIHT}
\label{subsec:restart-aiht}

Let \(\widehat\nu\) denote the latest restart time and let \(a_t=t-\widehat\nu+1\) be the local age. After a restart, AIHT uses local-age step sizes
\[
\eta_t
=
\begin{cases}
\alpha_1(a_t+b_1)^{-1/2}, & \mathsf{PHASE}=\mathrm{I},\\[2mm]
\{c_{\mathrm{eff}}(a_t+b_2)\}^{-1}, & \mathsf{PHASE}=\mathrm{II}.
\end{cases}
\]
Candidate sets, epoch lengths, projected-gradient statistics, and hard-thresholding steps are computed as in Algorithm~\ref{alg:aiht-adaptive}, but using only observations in the local buffer after the latest restart. The detector is evaluated only when \(a_t\ge2h\), so that both adjacent detection windows lie after the most recent restart.

\begin{algorithm}[htbp]
\caption{Restart-enabled candidate-restricted AIHT}
\label{alg:restart-mass-aiht}
\begin{algorithmic}[1]
\State \textbf{Input:} working sparsity \(s\), candidate size \(m\), detection window \(h\), threshold \(b_h=2C_D\sqrt{\log(pT/\delta)/h}\), and the AIHT tuning parameters from Algorithm~\ref{alg:aiht-adaptive}.
\State Initialize \(\beta_1=0\), \(\mathsf{PHASE}=\mathrm{I}\), latest restart time \(\widehat\nu=1\), local age \(a_1=1\), and the first local epoch.
\For{\(t=1,2,\ldots,T\)}
    \State Observe \((X_t,Y_t)\) and update the local buffer.
    \If{\(a_t\ge2h\)}
        \State Set \(\breve\beta_t=\beta_{t-2h}\) and compute \(D_t\) from \eqref{eq:score-contrast}.
        \If{\(D_t>b_h\)}
            \State Hard restart at \(t+1\): set \(\beta_{t+1}=0\), \(\mathsf{PHASE}=\mathrm{I}\), flush the local buffer, set \(\widehat\nu=t+1\), reset the local age, reset the local epoch counter and epoch length, and clear the candidate set and phase-transition diagnostic history.
            \State Continue to the next time point.
        \EndIf
    \EndIf
    \State Perform one candidate-restricted AIHT update as in Algorithm~\ref{alg:aiht-adaptive}, using the local age \(a_t\) in the step-size schedule.
\EndFor
\end{algorithmic}
\end{algorithm}

The hard restart convention is mainly for theoretical clarity. It prevents obsolete support coordinates from the previous segment from being carried into the new local cone. Warm-started restarts may be useful in practice, but their analysis requires an additional basin-of-attraction condition for the post-change support.

The next lemma is the concentration event used by the detector. It is stated conditionally on the frozen reference point, which is measurable before both detection windows.

\begin{lemma}[Detector score concentration]
\label{lem:score-detector-concentration}
Under {\normalfont (D1)}, with probability at least \(1-\delta\), the following holds simultaneously for all \(t\le T\). If \(\mathcal I_t^\pm\) is contained in segment \(k_\pm\), then
\[
\left\|
\bar g_t^\pm(\breve\beta_t)-\Psi_{k_\pm}(\breve\beta_t)
\right\|_\infty
\le
\frac{\lambda_h}{2}.
\]
\end{lemma}

\begin{lemma}[No false restart on clean adjacent windows]
\label{lem:no-false-clean}
On the event of Lemma~\ref{lem:score-detector-concentration}, if both adjacent windows \(\mathcal I_t^-\) and \(\mathcal I_t^+\) are contained in the same segment, then
\[
D_t\le \lambda_h<b_h.
\]
Thus the detector does not restart on clean adjacent windows.
\end{lemma}

\begin{lemma}[Detection after recovery]
\label{lem:detect-change}
Fix a changepoint \(\nu_k\) and set \(t_k=\nu_k+h-1\). Suppose the frozen reference point satisfies
\[
\|\breve\beta_{t_k}-\theta_{k-1}\|_2\le r_{\mathrm{det}}.
\]
Then, under {\normalfont (D2)}, on the event of Lemma~\ref{lem:score-detector-concentration},
\[
D_{t_k}>b_h.
\]
Consequently, the changepoint is detected no later than \(\nu_k+h\).
\end{lemma}

For the regret statement, define the clean sliding-window objective on segment \(k\). For \(t\in[\nu_k,\nu_{k+1}-1]\), let
\[
\mathcal I_{t,k}^{\mathrm{clean}}
=
\{i:\max(\nu_k,t-W+1)\le i\le t\},
\qquad
W_{t,k}^{\mathrm{clean}}=|\mathcal I_{t,k}^{\mathrm{clean}}|.
\]
Then
\[
Q_{t,k}^{\mathrm{clean}}(\beta)
=
\frac{1}{W_{t,k}^{\mathrm{clean}}}
\sum_{i\in\mathcal I_{t,k}^{\mathrm{clean}}}
\rho_\tau(Y_i-X_i^\top\beta).
\]
The dynamic regret relative to the piecewise-stationary target is
\[
\mathcal R_T^{\mathrm{dyn}}
=
\sum_{k=0}^{K}
\sum_{t=\nu_k}^{\nu_{k+1}-1}
\left\{
Q_{t,k}^{\mathrm{clean}}(\beta_t)
-
Q_{t,k}^{\mathrm{clean}}(\theta_k)
\right\}.
\]

Let \(\widehat\nu_0=1\), and for \(k\ge1\), let \(\widehat\nu_k\) denote the restart time associated with changepoint \(\nu_k\).

\begin{theorem}[Segmentwise recovery and dynamic regret]
\label{thm:segmentwise-restart}
Under {\normalfont (D1)}--{\normalfont (D3)}, restart-enabled candidate-restricted AIHT satisfies the following statements with probability at least \(1-\delta\).

\textnormal{(i) Restart delay.}
For every \(k=1,\ldots,K\), the restart associated with segment \(k\) occurs within \(h\) observations:
\[
\nu_k\le \widehat\nu_k\le \nu_k+h.
\]

\textnormal{(ii) Segmentwise recovery.}
After the restart associated with segment \(k\), the local run of AIHT obeys the same two-phase recovery behavior as in Theorem~\ref{thm:two-phase-convergence}, with \(\theta_k\) and \(S_k\) replacing \(\beta^\ast\) and \(S\). In particular, after the clean recovery period, during the local Phase~II refinement regime,
\[
\|\beta_t-\theta_k\|_2^2
\le
C\frac{\log(a_t^{(k)}+b_2)}{a_t^{(k)}+b_2},
\qquad
a_t^{(k)}=t-\widehat\nu_k+1.
\]

\textnormal{(iii) Dynamic regret.}
Let \(\mathcal R_{\mathrm I}\) denote the finite discovery-phase regret cost of one clean local run after restart, including buffer refill, burn-in, support discovery, and transition to the local regime. Then
\begin{equation}
\label{eq:dynamic-regret-main}
\mathcal R_T^{\mathrm{dyn}}
\le
KhM_{\mathrm{loc}}
+
(K+1)\mathcal R_{\mathrm I}
+
\frac{G_{\bar s}^2}{2c_{\mathrm{eff}}}
\sum_{k=0}^{K}
\log\{1+\nu_{k+1}-\widehat\nu_k\}.
\end{equation}
Consequently,
\begin{equation}
\label{eq:dynamic-regret-simplified}
\mathcal R_T^{\mathrm{dyn}}
\le
KhM_{\mathrm{loc}}
+
(K+1)\mathcal R_{\mathrm I}
+
\frac{G_{\bar s}^2}{2c_{\mathrm{eff}}}
(K+1)\log(T+1).
\end{equation}
\end{theorem}

\paragraph{Interpretation.}
The bound is the stationary regret theorem applied segment by segment, plus the detection-delay cost. Each changepoint contributes at most \(h\) observations before restart. After restart, the algorithm discards stale data, re-enters support discovery, rebuilds candidate sets, and then returns to stable local refinement once the new support is covered. Thus concept drift is handled by restarting the same two-phase mechanism rather than by solving a new batch problem.

\section{Simulation}
\label{sec:simulation}

We evaluate AIHT in three complementary experiments. The first experiment compares estimation accuracy and robustness against two standard online baselines: dense online SGD and Truncated Gradient (TG; \citealp{langford2009sparse}). The second experiment isolates the role of the hard-thresholding schedule by comparing AIHT with every-step and fixed-period hard-thresholding variants. The third experiment evaluates the restart-enabled version of AIHT under distributional shift.

Online SGD is included as a dense baseline that uses the same quantile subgradient information but does not enforce sparsity. TG is included as a soft-thresholding sparse online comparator. The schedule-ablation and distributional-shift experiments focus more directly on the hard-thresholding mechanism studied in Sections~\ref{sec:methodology}--\ref{sec:piecewise}.

\begin{figure*}[htbp]
    \centering
    \includegraphics[width=1\linewidth]{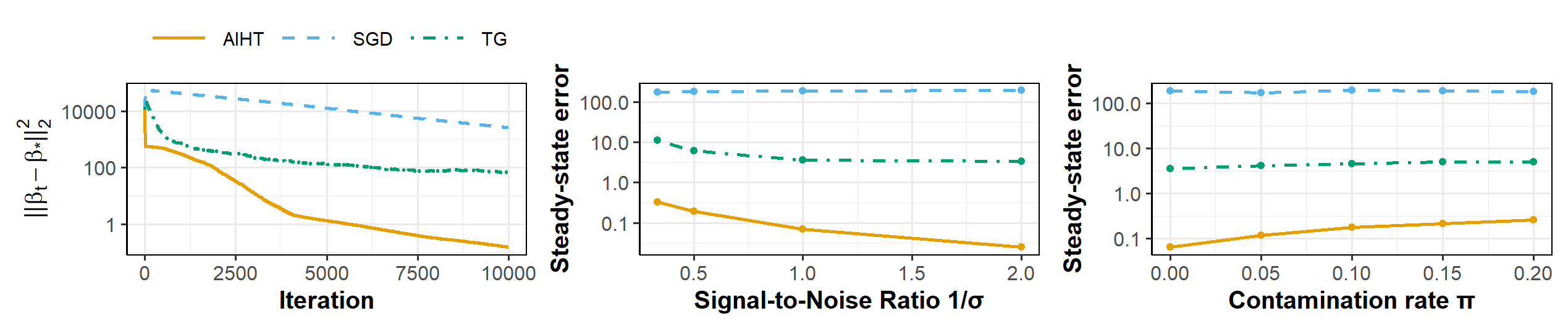}
    \caption{Performance of AIHT, Online SGD, and TG. Left: instantaneous estimation error
    \(\|\beta_t-\beta^\ast\|_2^2\) over \(10{,}000\) iterations. Middle: steady-state error
    (mean of the last \(2000\) iterates) as a function of the signal-to-noise ratio \(1/\sigma\).
    Right: steady-state error under Huber \(\pi\)-contamination. AIHT converges faster and attains
    substantially smaller steady-state error across all noise levels and contamination rates.}
    \label{fig:sim02}
\end{figure*}

\subsection{Experimental setup}
\label{subsec:simulation-setup}

Following model~\eqref{model}, we simulate predictors
\[
X_t\sim N(0,I_p)
\]
with ambient dimension \(p=2000\) and true sparsity \(s_0=20\). The nonzero coefficients are generated as
\[
\beta_i^\ast=5+\xi_i,\qquad i\le s_0,
\]
where \(\xi_i\) are small independent perturbations. We set the target quantile level to \(\tau=0.5\) and run each method for \(T_{\max}=10^4\) iterations.

The baseline noise distribution is
\[
\varepsilon_t\sim N(-q_\tau\sigma,\sigma^2),
\]
where \(q_\tau\) is the \(\tau\)-quantile of a standard normal random variable. This centering ensures that the \(\tau\)-quantile of \(\varepsilon_t\) is zero. Since \(\tau=0.5\) in the simulations, \(q_\tau=0\). To evaluate robustness to heavy-tailed contamination, we also consider Huber \(\pi\)-contamination: with probability \(1-\pi\), the noise is drawn from the baseline normal distribution, and with probability \(\pi\), it is drawn from a scaled Student-\(t_2\) distribution with scale factor \(5\sigma\). The contamination level is varied over
\[
\pi\in\{0,0.05,0.10,0.15,0.20\}.
\]

All methods use the same quantile subgradient information within each experiment. In the main accuracy and robustness comparison, the update uses the one-observation online quantile subgradient
\[
g_t(\beta_t)
=
-X_t\{\tau-\mathbf 1(Y_t\le X_t^\top\beta_t)\}.
\]
The schedule-ablation and distributional-shift experiments below use the sliding-window subgradient in \eqref{move_grad}, matching the implementation in Algorithm~\ref{alg:aiht-adaptive}.

For AIHT, we use \(\eta_0=5\), \(\alpha_1=\eta_0\), \(b_1=0\), \(\alpha_2=\eta_0\), \(b_2=50\), initial epoch length \(k_1=20\), and shrinkage factor \(\gamma=0.9\). Hard-thresholding methods use working sparsity \(s=40\). Online SGD uses the same quantile subgradient with step size \(\eta_0/\sqrt{t}\), but does not impose sparsity. TG is implemented as an online soft-thresholding baseline,
\[
\beta^{t+1}
=
\operatorname{shrink}(\beta^t-\eta_t g_t,\lambda_t),
\qquad
\lambda_t=0.05\,\eta_0/\sqrt{t},
\]
where the shrinkage operator is applied coordinatewise as
\[
[\operatorname{shrink}(z,\lambda)]_j
=
\operatorname{sign}(z_j)(|z_j|-\lambda)_+ .
\]

Performance is measured using the instantaneous estimation error
\[
\|\beta_t-\beta^\ast\|_2^2,
\]
together with a steady-state error defined as the average of \(\|\beta_t-\beta^\ast\|_2^2\) over the final \(2000\) iterations. The steady-state metric summarizes long-run accuracy after transient effects have diminished. We examine robustness under noise levels
\[
\sigma\in\{0.5,1,2,3\},
\]
which are displayed through the signal-to-noise ratio \(1/\sigma\), and under contamination levels
\[
\pi\in\{0,0.05,0.10,0.15,0.20\}.
\]

\subsection{Simulation results}
\label{subsec:simulation-results}

Figure~\ref{fig:sim02} reports the full error trajectory, the signal-to-noise sweep, and the contamination sweep. AIHT converges more rapidly than Online SGD and TG and attains substantially smaller steady-state error. The dense SGD baseline remains inaccurate because it does not exploit sparsity, while TG improves over dense SGD but is less effective than hard-thresholded sparse projection. Under heavier noise and increasing contamination, AIHT remains stable, illustrating the benefit of combining quantile subgradients with adaptive sparsity control.

We next isolate the algorithmic role of the hard-thresholding schedule. All methods in this ablation use the same sliding-window quantile subgradient, but differ in when the sparse projection is applied. The ablation uses
\[
p=400,\qquad s_0=15,\qquad s=30,\qquad W=60,\qquad T=5000.
\]
The covariates have an AR(1) covariance structure with
\[
\Sigma_{jk}=0.5^{|j-k|},
\]
and the noise follows a scaled Student-\(t_3\) distribution. AIHT uses the two-phase thresholding schedule from Algorithm~\ref{alg:aiht-adaptive}: in Phase~I it thresholds every \(50\) observations, and after the stabilization point \(T_0=1800\), it uses the Phase~II rule
\[
k_{j+1}=\max\{4,\lfloor0.88 k_j\rfloor\}.
\]
The fixed-period comparator keeps \(k_j\equiv 50\), while the every-step comparator applies \(H_s\) after each update. We fix \(T_0\) in this ablation to separate the effect of the thresholding schedule from the tuning of the phase-transition diagnostic.

For methods whose raw iterates are not exactly \(s\)-sparse at every time, the plotted post-thresholded error uses \(H_s(\beta_t)\). The cumulative excess check loss is computed as
\[
\sum_{u=1}^t
\left\{
\rho_\tau(Y_u-X_u^\top \widehat\beta_u)
-
\rho_\tau(Y_u-X_u^\top \beta^\ast)
\right\},
\]
where \(\widehat\beta_u\) denotes the estimator used for prediction at time \(u\); for the post-thresholded curves, \(\widehat\beta_u=H_s(\beta_u)\).

\begin{figure*}[htbp]
    \centering
    \includegraphics[width=1\linewidth]{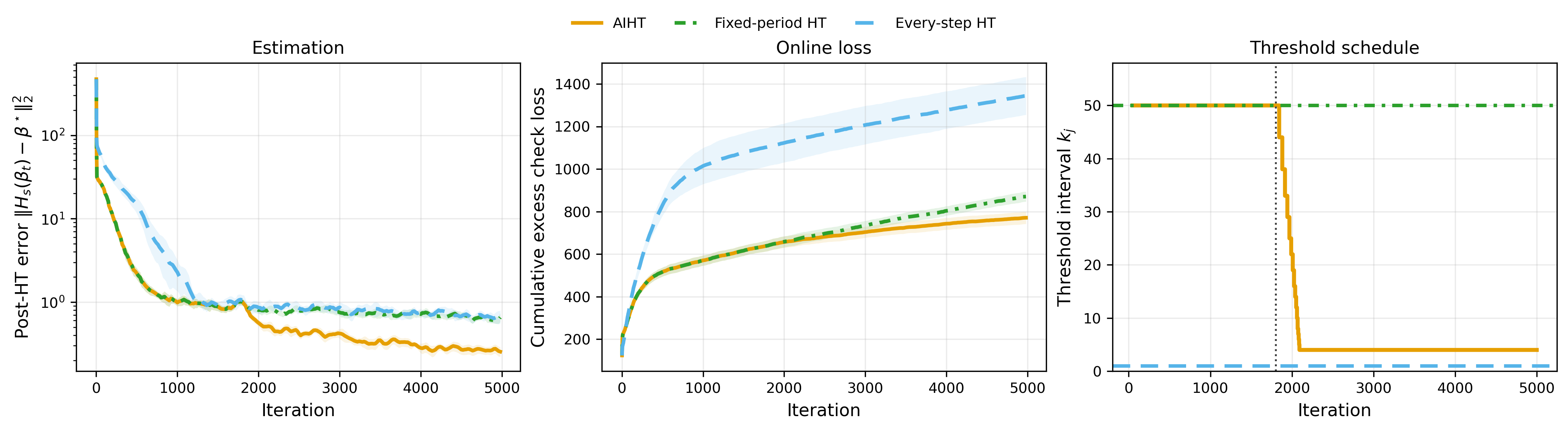}
    \caption{Threshold-scheduling ablation over \(12\) Monte Carlo replications.
    Left: post-thresholded estimation error.
    Middle: cumulative excess quantile check loss relative to the true sparse parameter.
    Right: threshold interval used by each method. Shaded bands show approximate \(95\%\)
    Monte Carlo standard-error intervals. AIHT keeps thresholding sufficiently delayed during
    support discovery and then increases projection frequency during refinement, leading to lower
    final error and smaller cumulative excess loss than fixed-period or every-step thresholding.}
    \label{fig:schedule-ablation}
\end{figure*}

Figure~\ref{fig:schedule-ablation} supports the epoch-mass mechanism developed in Section~\ref{subsec:support-entry-failure}. Every-step hard thresholding enforces sparsity aggressively but suppresses support entry, leading to larger error and larger cumulative excess loss. Fixed-period thresholding improves support discovery early in the run, but its projection interval does not adapt to the decaying step size. AIHT delays thresholding during support discovery and then increases the projection frequency during local refinement, producing the lowest final error among the hard-thresholding variants.

\paragraph{Distributional-shift experiment.}
Finally, we evaluate restart-enabled AIHT in a piecewise-stationary setting with one changepoint at \(t=2000\). The simulation uses \(p=200\), two disjoint supports of size \(s_0=10\), and \(T=4000\) observations. The restart detector is based on the adjacent-window score contrast in \eqref{eq:score-contrast}. For numerical stability, the implementation uses a persistence rule and cooldown period to avoid repeated restarts caused by short-lived score fluctuations; this practical modification is used only in the simulation, while the theory in Section~\ref{sec:piecewise} is stated for the simpler single-threshold detector.

All reported errors in this experiment are computed using the post-thresholded sparse estimator. We compare Restart-AIHT with four baselines: AIHT without restart, global periodic IHT, global every-step IHT, and restart-enabled every-step IHT.

\FloatBarrier

\begin{figure*}[htbp]
    \centering
    \begin{minipage}[t]{0.49\textwidth}
        \centering
        \includegraphics[width=\linewidth]{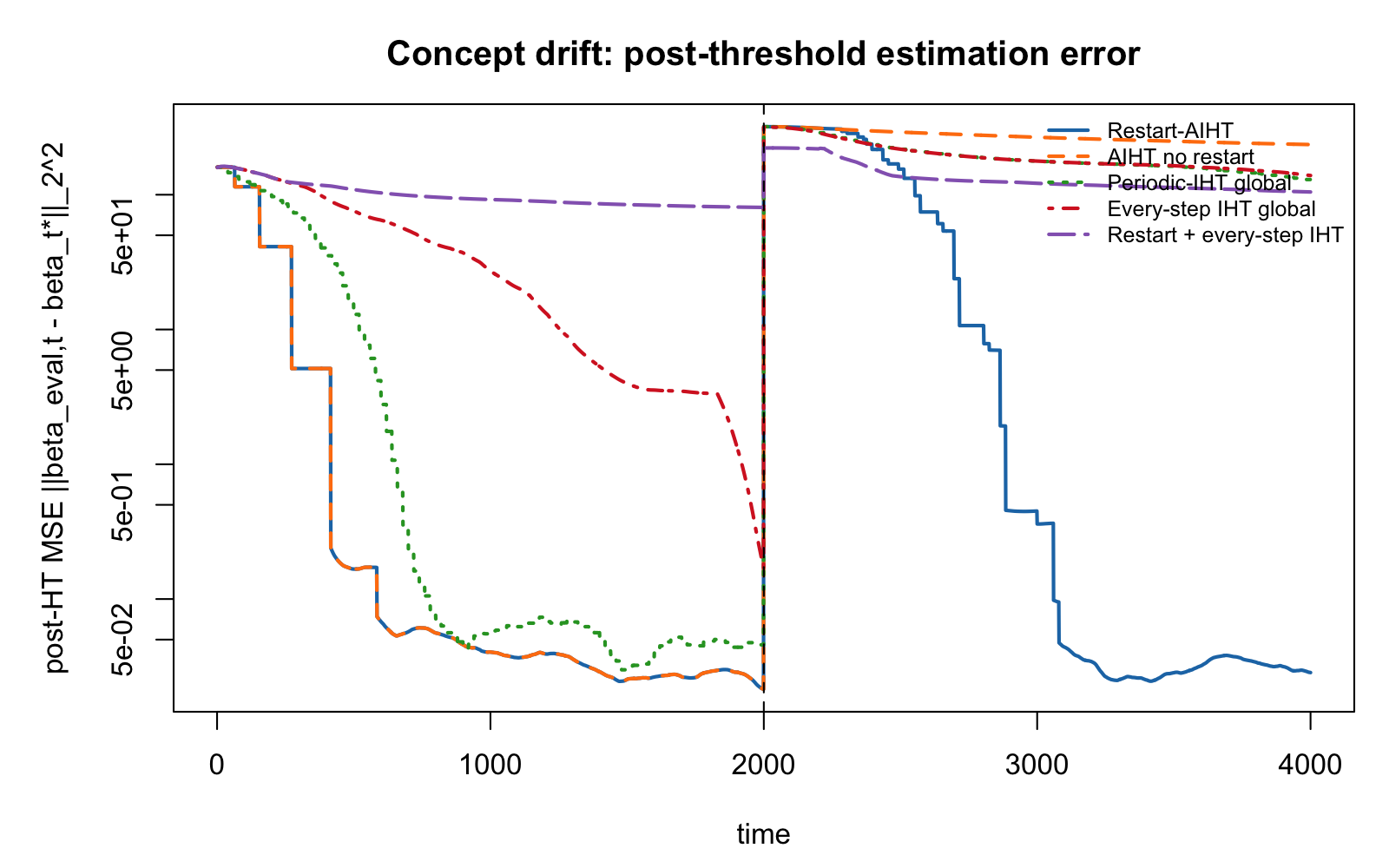}
    \end{minipage}\hfill
    \begin{minipage}[t]{0.49\textwidth}
        \centering
        \includegraphics[width=\linewidth]{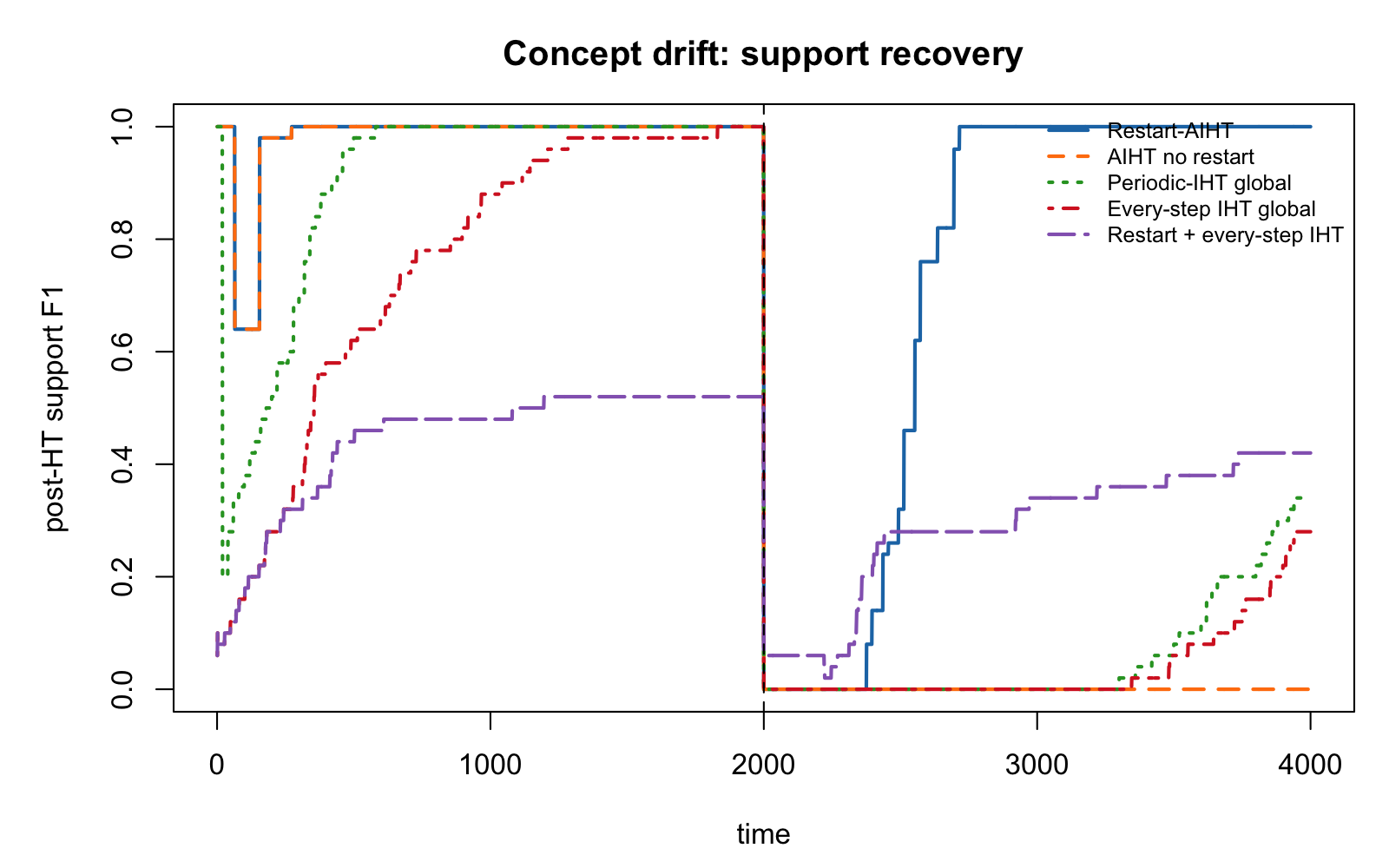}
    \end{minipage}
    \vspace{0.5em}

    \includegraphics[width=0.66\textwidth]{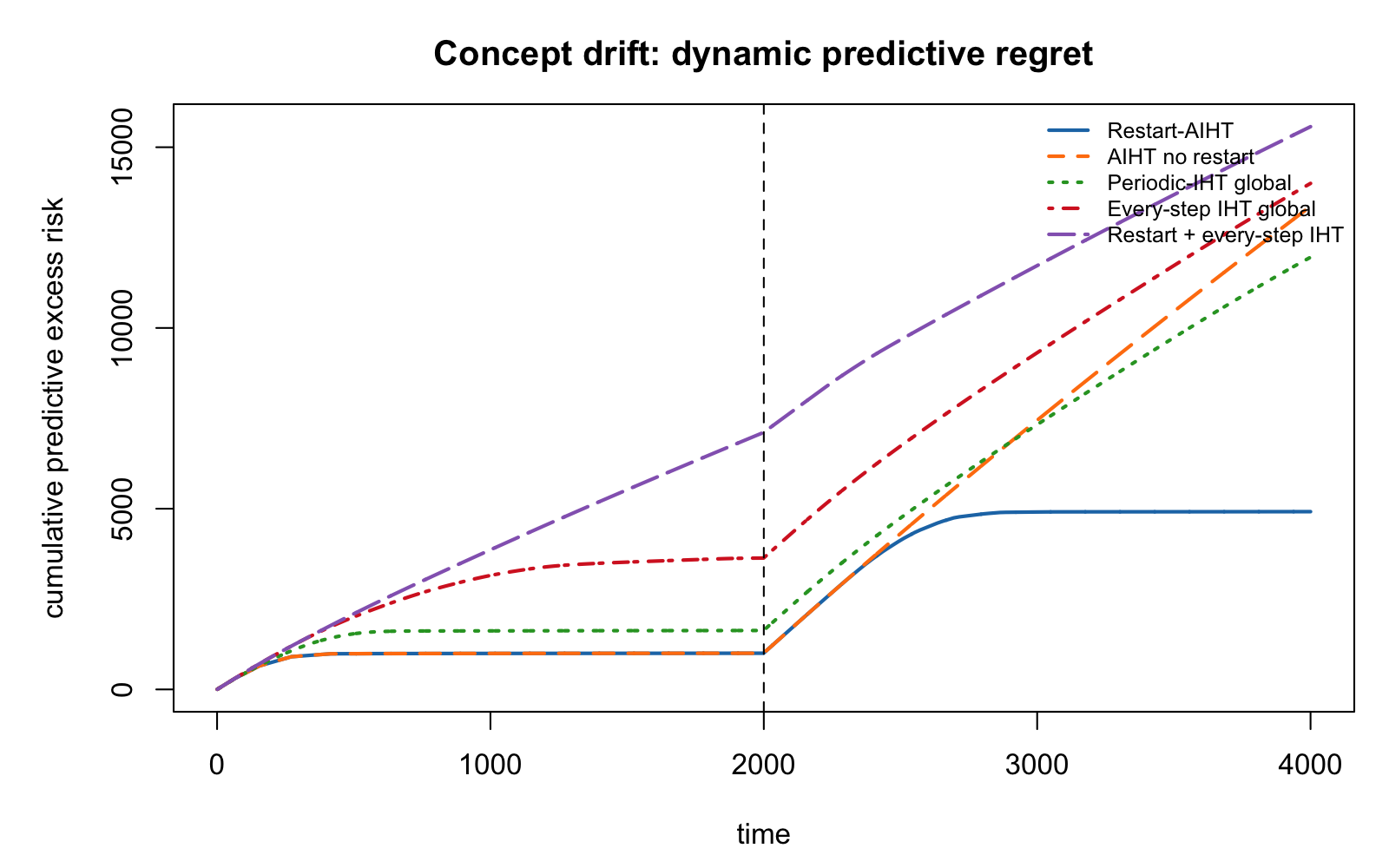}
    \caption{Concept-drift experiment with a changepoint at \(t=2000\).
    Top left: post-thresholded estimation error. Top right: support recovery.
    Bottom: cumulative dynamic predictive regret. Restart-AIHT detects the change,
    re-enters the discovery phase, and recovers the new support. In contrast, methods
    without restart remain anchored to the old support, while every-step thresholding after
    restart recovers only partially because it again suppresses support entry too aggressively.}
    \label{fig:concept-drift}
\end{figure*}

Figure~\ref{fig:concept-drift} complements the segmentwise guarantee in Theorem~\ref{thm:segmentwise-restart}. Before the changepoint, Restart-AIHT behaves like stationary AIHT. After the change, the restart flushes the stale buffer and resets the local clock, so the method pays a finite recovery cost rather than accumulating persistent error. Across five replications, the average detection delay is about \(265\) observations, the final support F1 is \(1.00\), and the final post-thresholded MSE is about \(0.03\). The no-restart and global-thresholding baselines have much larger final errors because their sparse trajectories continue to reflect the obsolete regime.

	\section{Conclusion}
	This paper proposed AIHT, an online algorithmic framework for high-dimensional sparse regression. By combining stochastic subgradient updates with an adaptive schedule for hard thresholding and step sizes, the method stabilizes the iterates, exploits local curvature when it appears, and maintains accurate sparsity throughout the learning process. For online quantile regression, our analysis shows that candidate-restricted AIHT has a two-phase convergence behavior and logarithmic refinement-phase regret for the sliding-window objective, with total regret \(C_{\mathrm I}(T_0)+C_{\mathrm{burn}}+O(\log T)\), while empirical results confirm its advantages over online SGD and truncated gradient baselines across a range of noise regimes and threshold-scheduling ablations. A promising direction for future work is to develop formal guarantees for additional losses, such as squared, Huber, or generalized linear losses, under the same restricted-curvature and gradient-leakage framework.

	\bibliographystyle{apalike}
	\bibliography{references}
	
	\appendix
	\FloatBarrier
	\onecolumn

	\section{Theoretical Proofs}
	% ---------------------------
	% 附录 B：独立的定理体系
	% ---------------------------
	\newtheorem{Btheorem}{Theorem}[section]
	\newtheorem{Blemma}[Btheorem]{Lemma}
	\newtheorem{Bproposition}[Btheorem]{Proposition}
	\newtheorem{Bcorollary}[Btheorem]{Corollary}
	
	% 编号改成 "B.1", "B.2" ...
	\renewcommand{\theBtheorem}{B.\arabic{Btheorem}}
	\renewcommand{\theBlemma}{B.\arabic{Btheorem}}
	\renewcommand{\theBproposition}{B.\arabic{Btheorem}}
	\renewcommand{\theBcorollary}{B.\arabic{Btheorem}}
	
	\label{appendix:theory}
	For clarity and completeness, we restate below the assumptions previously introduced in the main text and used throughout the proof of the theorem. For simplicity, throughout this section we write average gradient $\bar g_j$ just as $g_j$.\\
We work under the assumptions and post-burn-in setup stated in Section~\ref{subsec:assumptions-curvature}.
	\subsection*{Proofs for Section~\ref{sec:methodology}}
	
	\begin{proof}[Proof of Proposition~\ref{prop:one-step-barrier}]
		By definition, \(H_s(u_{t+1})\) keeps the \(s\) largest coordinates of
		\(u_{t+1}\) in absolute value. Since \(m\notin\widehat S_t\), we have
		\(\beta_{t,m}=0\), and therefore
		\[
		u_{t+1,m}
		=
		-\eta_tg_{t,m}(\beta_t).
		\]
		Hence \(|u_{t+1,m}|=\eta_t|g_{t,m}(\beta_t)|\). By the definition of
		\(\lambda_t^{(-m)}\), and ignoring ties, coordinate \(m\) is selected by
		\(H_s\) if and only if
		\[
		|u_{t+1,m}|>\lambda_t^{(-m)},
		\]
		which gives \eqref{eq:one-step-entry}. If
		\(|g_{t,m}(\beta_t)|\le G_m\), \(\lambda_t^{(-m)}\ge\lambda_0t^{-r}\), and
		\(\eta_t=\alpha t^{-q}\), then
		\[
		\eta_t|g_{t,m}(\beta_t)|
		\le
		\alpha G_m t^{-q}.
		\]
		For \(t^{q-r}\ge\alpha G_m/\lambda_0\), the right-hand side is at most
		\(\lambda_0t^{-r}\le\lambda_t^{(-m)}\), so the entry condition fails. Since
		\(q>r\), this is equivalent to
		\(t\ge(\alpha G_m/\lambda_0)^{1/(q-r)}\).
	\end{proof}
	
	\begin{proof}[Proof of Proposition~\ref{prop:wrong-fixed-point}]
		For the quadratic loss,
		\[
		\nabla L(\beta)=\Sigma(\beta-\beta^\ast).
		\]
		At \(\bar\beta=(0,\rho\theta)^\top\),
		\[
		\bar\beta-\beta^\ast=(-\theta,\rho\theta)^\top,
		\]
		and therefore
		\[
		\nabla L(\bar\beta)
		=
		\begin{pmatrix}
		1 & \rho\\
		\rho & 1
		\end{pmatrix}
		\begin{pmatrix}
		-\theta\\
		\rho\theta
		\end{pmatrix}
		=
		\begin{pmatrix}
		-(1-\rho^2)\theta\\
		0
		\end{pmatrix}.
		\]
		Thus
		\[
		\bar\beta-\eta\nabla L(\bar\beta)
		=
		\begin{pmatrix}
		\eta(1-\rho^2)\theta\\
		\rho\theta
		\end{pmatrix}.
		\]
		Because \(s=1\), \(H_1\) keeps the coordinate with larger absolute value.
		Since \(\theta>0\), the second coordinate is retained whenever
		\(\rho\theta>\eta(1-\rho^2)\theta\), equivalently
		\(\eta<\rho/(1-\rho^2)\). In that case the update returns exactly
		\(\bar\beta\). For \(\eta_t=\alpha t^{-q}\), the same inequality holds once
		\(t>\{\alpha(1-\rho^2)/\rho\}^{1/q}\), proving the absorbing
		wrong-support claim.
	\end{proof}
	
	\begin{proof}[Proof of Proposition~\ref{prop:fixed-period-vanishing-mass}]
		For fixed \(K\),
		\[
		B_j=\alpha\sum_{\ell=0}^{K-1}(\tau_j+\ell)^{-q}.
		\]
		If \(0<q<1\), then uniformly over \(\ell=0,\ldots,K-1\),
		\[
		(\tau_j+\ell)^{-q}
		=
		\tau_j^{-q}\{1+O(K/\tau_j)\},
		\]
		and hence
		\[
		B_j=\alpha K\tau_j^{-q}\{1+o(1)\}\to0.
		\]
		If \(q=1\), comparison with the integral of \(1/x\) gives
		\[
		\sum_{\ell=0}^{K-1}\frac{1}{\tau_j+\ell}
		=
		\log\left(1+\frac{K}{\tau_j}\right)+o(\tau_j^{-1})
		=
		K\tau_j^{-1}\{1+o(1)\},
		\]
		so again \(B_j\to0\). Therefore every fixed positive entry threshold
		\(B_{\mathrm{entry}}\) eventually exceeds \(B_j\).
	\end{proof}
	
	\begin{proof}[Proof of Proposition~\ref{prop:epoch-mass-cone-barrier}]
		Let
		\[
		a_t=\|\Delta_{t,S^c}\|_1,\qquad
		b_t=\|\Delta_{t,S}\|_1,\qquad
		r_t=\frac{a_t}{b_t+\zeta}.
		\]
		The update and the triangle inequality give
		\[
		a_{t+1}\le a_t+\eta_tL_{\mathrm{off}}(b_t+\zeta),
		\]
		while
		\[
		b_{t+1}+\zeta
		\ge
		(1-\eta_tL_{\mathrm{on}})(b_t+\zeta).
		\]
		Since \(\eta_tL_{\mathrm{on}}\le1/2\),
		\[
		r_{t+1}
		\le
		\frac{r_t+\eta_tL_{\mathrm{off}}}{1-\eta_tL_{\mathrm{on}}}
		\le
		(r_t+\eta_tL_{\mathrm{off}})(1+2\eta_tL_{\mathrm{on}}).
		\]
		Expanding and using \(\eta_tL_{\mathrm{on}}\le1/2\) yields
		\[
		r_{t+1}
		\le
		r_t+2\eta_tL_{\mathrm{off}}+2\eta_tL_{\mathrm{on}}r_t.
		\]
		As long as \(r_t\le c_{\mathrm{large}}\), this proves
		\eqref{eq:ratio-recursion-method}. Summing over the epoch gives
		\[
		r_t
		\le
		c_{\mathrm{small}}
		+
		2\{L_{\mathrm{off}}+c_{\mathrm{large}}L_{\mathrm{on}}\}B_j,
		\]
		and \eqref{eq:Bcone-method} is sufficient for \(r_t\le c_{\mathrm{large}}\).
		
		For the converse, set \(D=b_{\tau_j}+\zeta>0\), choose
		\(a_{\tau_j}=c_{\mathrm{small}}D\), and let \(h\) be a unit
		\(\ell_1\)-vector supported on \(S^c\) and aligned with the off-support
		error. Take
		\[
		g_{t,S}(\beta_t)=0,\qquad
		g_{t,S^c}(\beta_t)=-L_{\mathrm{off}}D\,h.
		\]
		Then \eqref{eq:method-gradient-guard} holds with \(L_{\mathrm{on}}=0\), and
		\(b_t+\zeta=D\) remains constant. Moreover
		\[
		r_t
		=
		c_{\mathrm{small}}
		+
		L_{\mathrm{off}}\sum_{u=\tau_j}^{t-1}\eta_u.
		\]
		If \(B_j>(c_{\mathrm{large}}-c_{\mathrm{small}})/L_{\mathrm{off}}\), then
		\(r_t>c_{\mathrm{large}}\) for some time in the epoch. Thus a uniform
		cone-invariance statement requires an upper bound on epoch mass.
	\end{proof}
	
	\subsection*{Restricted strong convexity}
	
\begin{Blemma}[Sparse restricted curvature for sliding-window quantile loss]
\label{app:lem:rsc}
Under {\normalfont (A1)}--{\normalfont (A3)}, there exist constants \(a^\ast,c^\ast>0\) such that, with high probability, uniformly over \(t\ge W\) and all \(\Delta\) satisfying \(\|\Delta\|_0\le \bar s\) and \(\|\Delta\|_2\le R\),
\[
\bigl\langle
\bar g_t(\beta^\ast+\Delta)-\bar g_t(\beta^\ast),\Delta
\bigr\rangle
\ge
a^\ast\|\Delta\|_2^2
-
c^\ast\sqrt{\frac{\bar s\log(pT/\delta)}{W}}\|\Delta\|_2.
\]
Consequently,
\[
\bigl\langle
\bar g_t(\beta^\ast+\Delta)-\bar g_t(\beta^\ast),\Delta
\bigr\rangle
\ge
\mu\|\Delta\|_2^2
-
C_{\mathrm{rsc}}\frac{\bar s\log(pT/\delta)}{W}.
\]
\end{Blemma}
	
	This lemma establishes the restricted strong convexity (RSC) condition on the set 
	$\{\|\Delta\|_2 \le 1\} \cap \mathcal{C}$. 
	In fact, this result is exactly Theorem~1 in \citet{wang2024analysis}. 
	Hence, we do not reprove it here and simply invoke the established lemma for our analysis.\\
	\begin{Bproposition}
		\label{prop:bounded-score-truth} Assume \textnormal{(A1)} and the conditional quantile condition $Q_\tau(\varepsilon_i\mid X_i)=0$ for all $i$. Then there exists a constant $C_1>0$ such that for any $\delta\in(0,1)$,
		with probability at least $1-\delta$,
		\[
		\max_{t=W,\dots,T}
		\|\bar g_t(\beta^\ast)\|_\infty
		\le
		C_1\sqrt{\frac{\log(pT/\delta)}{W}}.
		\]
	\end{Bproposition}
	\begin{proof}
		For each coordinate $j=1,\dots,p$, define
		\[
		Z_{ij}
		:=
		X_{ij}\bigl(\tau-\mathbf{1}\{\varepsilon_i\le 0\}\bigr).
		\]
		Since
		\[
		Y_i=X_i^\top\beta^\ast+\varepsilon_i,
		\]
		we have
		\[
		\bar g_{t,j}(\beta^\ast)
		=
		-\frac{1}{W}\sum_{i=t-W+1}^t Z_{ij}.
		\]
		
		By the conditional quantile assumption,
		\[
		\mathbb P(\varepsilon_i\le 0\mid X_i)=\tau,
		\]
		and therefore
		\[
		\mathbb E[Z_{ij}\mid X_i]
		=
		X_{ij}\,\mathbb E\bigl[\tau-\mathbf{1}\{\varepsilon_i\le 0\}\mid X_i\bigr]
		=
		0.
		\]
		Hence $\mathbb E[Z_{ij}]=0$.
		
		Under \textnormal{(A1)}, each coordinate $X_{ij}$ is sub-Gaussian with
		$\|X_{ij}\|_{\psi_2}\le K$, while
		\[
		\bigl|\tau-\mathbf{1}\{\varepsilon_i\le 0\}\bigr|\le 1.
		\]
		It follows that $Z_{ij}$ is centered sub-Gaussian and
		\[
		\|Z_{ij}\|_{\psi_2}\le C K
		\]
		for some universal constant $C>0$. Therefore, for each fixed $t$ and $j$,
		a standard tail bound for averages of independent centered sub-Gaussian random variables gives
		\[
		\mathbb P\!\left(
		\left|
		\frac{1}{W}\sum_{i=t-W+1}^t Z_{ij}
		\right|\ge u
		\right)
		\le
		2\exp\!\left(-\frac{cWu^2}{K^2}\right)
		\]
		for some constant $c>0$.
		
		Applying a union bound over all $j=1,\dots,p$ and all $t=W,\dots,T$ yields
		\[
		\mathbb P\!\left(
		\max_{t=W,\dots,T}\max_{1\le j\le p}
		|\bar g_{t,j}(\beta^\ast)|\ge u
		\right)
		\le
		2pT\exp\!\left(-\frac{cWu^2}{K^2}\right).
		\]
		Choosing
		\[
		u=C_1\sqrt{\frac{\log(pT/\delta)}{W}}
		\]
		with $C_1$ sufficiently large makes the right-hand side at most $\delta$, which proves the claim.
	\end{proof}
	
	\subsection*{Gradient Leakage for Quantile Loss}

	\begin{proof}
		Let
		\[
		Y_i=X_i^\top\beta^\ast+\varepsilon_i,
		\qquad
		\Delta=\beta-\beta^\ast.
		\]
		Then the empirical window subgradient at $\beta$ is
		\[
		g_t(\beta)
		=
		-\frac{1}{W}\sum_{i=t-W+1}^t
		X_i\bigl(\tau-\mathbf 1\{\varepsilon_i\le X_i^\top\Delta\}\bigr),
		\]
		while at the truth $\beta^\ast$,
		\[
		g_t(\beta^\ast)
		=
		-\frac{1}{W}\sum_{i=t-W+1}^t
		X_i\bigl(\tau-\mathbf 1\{\varepsilon_i\le 0\}\bigr).
		\]
		Hence
		\begin{equation}
			\label{eq:g-diff-start}
			g_t(\beta)-g_t(\beta^\ast)
			=
			\frac{1}{W}\sum_{i=t-W+1}^t
			X_i
			\Bigl[
			\mathbf 1\{\varepsilon_i\le X_i^\top\Delta\}
			-
			\mathbf 1\{\varepsilon_i\le 0\}
			\Bigr].
		\end{equation}
		
		For notational convenience, re-index the window as $i=1,\dots,W$.
		Then
		\[
		g_t(\beta)-g_t(\beta^\ast)
		=
		\frac{1}{W}\sum_{i=1}^W
		X_i
		\Bigl[
		\mathbf 1\{\varepsilon_i\le X_i^\top\Delta\}
		-
		\mathbf 1\{\varepsilon_i\le 0\}
		\Bigr].
		\]
		
		We begin with the decomposition
		\begin{equation}
			\label{eq:main-triangle}
			\|g_{t,S^c}(\beta)\|_\infty
			\le
			\|(g_t(\beta)-g_t(\beta^\ast))_{S^c}\|_\infty
			+
			\|g_{t,S^c}(\beta^\ast)\|_\infty.
		\end{equation}
		The two terms on the right-hand side are controlled separately.
		
		\medskip
		By Proposition~\ref{prop:bounded-score-truth} in its fixed-$t$ form, there exists a constant
		$C_0>0$ such that with probability at least $1-\delta/3$,
		\begin{equation}
			\label{eq:truth-score-bound}
			\|g_{t,S^c}(\beta^\ast)\|_\infty
			\le
			C_0\sqrt{\frac{\log(p/\delta)}{W}}.
		\end{equation}
		Define
		\[
		\zeta_i(\Delta)
		:=
		\mathbf 1\{\varepsilon_i\le X_i^\top\Delta\}
		-
		\mathbf 1\{\varepsilon_i\le 0\}.
		\]
		Then \eqref{eq:g-diff-start} becomes
		\[
		g_t(\beta)-g_t(\beta^\ast)
		=
		\frac{1}{W}\sum_{i=1}^W X_i\,\zeta_i(\Delta).
		\]
		Conditioning on $X_1,\dots,X_W$, we have
		\[
		\mathbb E[\zeta_i(\Delta)\mid X_i]
		=
		F_{\varepsilon|X}(X_i^\top\Delta\mid X_i)
		-
		F_{\varepsilon|X}(0\mid X_i).
		\]
		By the mean-value theorem, for each $i$ there exists a random point
		$\xi_i$ between $0$ and $X_i^\top\Delta$ such that
		\[
		F_{\varepsilon|X}(X_i^\top\Delta\mid X_i)
		-
		F_{\varepsilon|X}(0\mid X_i)
		=
		f_{\varepsilon|X}(\xi_i\mid X_i)\,X_i^\top\Delta.
		\]
		By {\normalfont (A2)}, \(f_{\varepsilon|X}(u\mid X_i)\le M_1\) for all \(u\in\mathbb R\). Hence
\[
\bigl|
F_{\varepsilon|X}(X_i^\top\Delta\mid X_i)
-
F_{\varepsilon|X}(0\mid X_i)
\bigr|
\le
M_1|X_i^\top\Delta|.
\]
		Therefore,
		\begin{align*}
			\Bigl\|
			\mathbb E\bigl[(g_t(\beta)-g_t(\beta^\ast))_{S^c}\mid X\bigr]
			\Bigr\|_\infty
			&=
			\Bigl\|
			\frac{1}{W}\sum_{i=1}^W
			X_{i,S^c}\,
			\mathbb E[\zeta_i(\Delta)\mid X_i]
			\Bigr\|_\infty \\
			&\le
			M_0
			\Bigl\|
			\frac{1}{W}\sum_{i=1}^W
			X_{i,S^c} X_i^\top \Delta
			\Bigr\|_\infty \\
			&=
			M_0
			\Bigl\|
			\frac{1}{W}X_{S^c}^\top X\Delta
			\Bigr\|_\infty.
		\end{align*}
		Since \(\operatorname{supp}(\Delta)\subseteq J\), we have
\[
\frac{1}{W}X_{S^c}^\top X\Delta
=
\frac{1}{W}X_{S^c}^\top X_J\Delta_J.
\]
Therefore, by the sparse Gram control in {\normalfont (A4)},
\[
\left\|
\frac{1}{W}X_{S^c}^\top X\Delta
\right\|_\infty
\le
\left\|
\frac{1}{W}X_{S^c}^\top X_J
\right\|_{\infty\to\infty}
\|\Delta_J\|_1
\le
C_\infty\|\Delta_J\|_1.
\]
Because \(S\subseteq J\) and \(\Delta\in\Gamma(c,S)\),
\[
\|\Delta_J\|_1
\le
\|\Delta_S\|_1+\|\Delta_{S^c}\|_1
\le
(1+c)\|\Delta_S\|_1.
\]
		Hence
		\begin{align*}
			\Bigl\|
			\frac{1}{W}X_{S^c}^\top X\Delta
			\Bigr\|_\infty
			&\le
			\Bigl\|
			\frac{1}{W}X_{S^c}^\top X_S
			\Bigr\|_{\infty\to\infty}
			\|\Delta_S\|_1
			+
			\Bigl\|
			\frac{1}{W}X_{S^c}^\top X_{S^c}
			\Bigr\|_{\infty\to\infty}
			\|\Delta_{S^c}\|_1.
		\end{align*}
		On the high-probability event from {\normalfont (A4)}, both operator norms are bounded by
		$C_\infty$. Therefore,
		\[
		\Bigl\|
		\mathbb E\bigl[(g_t(\beta)-g_t(\beta^\ast))_{S^c}\mid X\bigr]
		\Bigr\|_\infty
		\le
		M_1 C_\infty\bigl(\|\Delta_S\|_1+\|\Delta_{S^c}\|_1\bigr).
		\]
		Since $\Delta\in \Gamma_H(c,S)$, we have
		\[
		\|\Delta_{S^c}\|_1\le c\|\Delta_S\|_1,
		\]
		so
		\begin{equation}
			\label{eq:conditional-mean-bound}
			\Bigl\|
			\mathbb E\bigl[(g_t(\beta)-g_t(\beta^\ast))_{S^c}\mid X\bigr]
			\Bigr\|_\infty
			\le
			\rho_0\,\|\Delta_S\|_1,
			\qquad
			\rho_0:=M_0C_\infty(1+c).
		\end{equation}
		Define the centered variables
		\[
		U_{ij}
		:=
		X_{ij}
		\Bigl(
		\zeta_i(\Delta)-\mathbb E[\zeta_i(\Delta)\mid X_i]
		\Bigr),
		\qquad j\in S^c.
		\]
		Then
		\[
		(g_t(\beta)-g_t(\beta^\ast))_{S^c}
		-
		\mathbb E[(g_t(\beta)-g_t(\beta^\ast))_{S^c}\mid X]
		=
		\frac{1}{W}\sum_{i=1}^W U_i,
		\]
		where $U_i=(U_{ij})_{j\in S^c}$.
		
		For each fixed $j\in S^c$, the random variables $\{U_{ij}\}_{i=1}^W$ are independent,
		centered given $X$, and satisfy
		\[
		|\,\zeta_i(\Delta)-\mathbb E[\zeta_i(\Delta)\mid X_i]\,|\le 2.
		\]
		Since $X_{ij}$ is sub-Gaussian by {\normalfont (A1)}, it follows that each $U_{ij}$ is centered
		sub-Gaussian with parameter bounded by a constant multiple of $K$. Hence there exists
		a constant $c_0>0$ such that for every $u>0$,
		\[
		\mathbb P\!\left(
		\left|
		\frac{1}{W}\sum_{i=1}^W U_{ij}
		\right|
		\ge u
		\right)
		\le
		2\exp\!\left(-\frac{c_0Wu^2}{K^2}\right).
		\]
		Applying a union bound over $j\in S^c$ and choosing
		\[
		u=C_0\sqrt{\frac{\log(p/\delta)}{W}}
		\]
		with $C_0$ sufficiently large gives
		\begin{equation}
			\label{eq:fluctuation-bound}
			\Bigl\|
			(g_t(\beta)-g_t(\beta^\ast))_{S^c}
			-
			\mathbb E[(g_t(\beta)-g_t(\beta^\ast))_{S^c}\mid X]
			\Bigr\|_\infty
			\le
			C_0\sqrt{\frac{\log(p/\delta)}{W}}
		\end{equation}
		with probability at least $1-\delta/3$.
		
		Combining \eqref{eq:conditional-mean-bound} and \eqref{eq:fluctuation-bound}, we obtain with probability at least
		$1-2\delta/3$,
		\begin{equation}
			\label{eq:diff-bound-final}
			\|(g_t(\beta)-g_t(\beta^\ast))_{S^c}\|_\infty
			\le
			\rho_0\,\|\Delta_S\|_1
			+
			C_0\sqrt{\frac{\log(p/\delta)}{W}}.
		\end{equation}
		Substituting \eqref{eq:truth-score-bound} and \eqref{eq:diff-bound-final}
		into \eqref{eq:main-triangle}, and enlarging constants if necessary, yields
		\[
		\|g_{t,S^c}(\beta)\|_\infty
		\le
		\rho\,\|\Delta_S\|_1
		+
		C_0\sqrt{\frac{\log(p/\delta)}{W}}
		\]
		with probability at least $1-\delta$, where $\rho>0$ depends only on
		$K$, $(m_0,M_0)$, $c$, and the Gram bound constant in {\normalfont (A4)}.
		This proves the lemma.
	\end{proof}

	\subsection*{On–support gradient bound}
	
\begin{Blemma}[On-support coordinatewise gradient bound]
\label{app:lem:on-support-grad}
Assume \textnormal{(A1)--(A4)} and let \(t\ge W\). Fix any set \(J\subseteq\{1,\ldots,p\}\) such that \(S\subseteq J\) and \(|J|\le\bar s=s+m+s_0\). For any \(\beta\) with \(\Delta=\beta-\beta^\ast\), \(\operatorname{supp}(\Delta)\subseteq J\), \(\Delta\in\Gamma(c,S)\), and \(\|\Delta\|_2\le R\), there exist constants \(L_g>0\) and \(C_0>0\) such that, for any \(\delta\in(0,1)\) and \(W\ge C\log(p/\delta)\), with probability at least \(1-\delta\),
\[
\|\bar g_{t,S}(\beta)\|_\infty
\le
L_g\|\Delta_S\|_1
+
C_0\sqrt{\frac{\log(p/\delta)}{W}}.
\]
\end{Blemma}
	
	\begin{proof}
		As before, write
		\[
		Y_i=X_i^\top\beta^\ast+\varepsilon_i,
		\qquad
		\Delta=\beta-\beta^\ast,
		\]
		and re-index the window as $i=1,\dots,W$ for convenience. Then
		\[
		g_t(\beta)-g_t(\beta^\ast)
		=
		\frac{1}{W}\sum_{i=1}^W
		X_i
		\Bigl[
		\mathbf 1\{\varepsilon_i\le X_i^\top\Delta\}
		-
		\mathbf 1\{\varepsilon_i\le 0\}
		\Bigr].
		\]
		We begin with the decomposition
		\begin{equation}
			\label{eq:on-support-triangle}
			\|g_{t,S}(\beta)\|_\infty
			\le
			\|(g_t(\beta)-g_t(\beta^\ast))_S\|_\infty
			+
			\|g_{t,S}(\beta^\ast)\|_\infty.
		\end{equation}
		By the fixed-$t$ version of Proposition~\ref{prop:bounded-score-truth}, there exists
		a constant $C_0>0$ such that with probability at least $1-\delta/3$,
		\begin{equation}
			\label{eq:on-support-truth}
			\|g_{t,S}(\beta^\ast)\|_\infty
			\le
			C_0\sqrt{\frac{\log(p/\delta)}{W}}.
		\end{equation}
		Define
		\[
		\zeta_i(\Delta)
		:=
		\mathbf 1\{\varepsilon_i\le X_i^\top\Delta\}
		-
		\mathbf 1\{\varepsilon_i\le 0\}.
		\]
		Then
		\[
		g_t(\beta)-g_t(\beta^\ast)
		=
		\frac{1}{W}\sum_{i=1}^W X_i\,\zeta_i(\Delta).
		\]
		Conditioning on $X_1,\dots,X_W$, we have
		\[
		\mathbb E[\zeta_i(\Delta)\mid X_i]
		=
		F_{\varepsilon|X}(X_i^\top\Delta\mid X_i)
		-
		F_{\varepsilon|X}(0\mid X_i).
		\]
		By the mean-value theorem, for each $i$ there exists $\xi_i$ between $0$ and
		$X_i^\top\Delta$ such that
		\[
		F_{\varepsilon|X}(X_i^\top\Delta\mid X_i)-F_{\varepsilon|X}(0\mid X_i)
		=
		f_{\varepsilon|X}(\xi_i\mid X_i)\,X_i^\top\Delta.
		\]
		Under \textnormal{(A2)}, $|f_{\varepsilon|X}(\xi_i\mid X_i)|\le M_0$, hence
		\[
		\bigl|
		F_{\varepsilon|X}(X_i^\top\Delta\mid X_i)-F_{\varepsilon|X}(0\mid X_i)
		\bigr|
		\le
		M_0 |X_i^\top\Delta|.
		\]
		Therefore,
		\begin{align*}
			\Bigl\|
			\mathbb E[(g_t(\beta)-g_t(\beta^\ast))_S\mid X]
			\Bigr\|_\infty
			&=
			\Bigl\|
			\frac{1}{W}\sum_{i=1}^W
			X_{i,S}\,\mathbb E[\zeta_i(\Delta)\mid X_i]
			\Bigr\|_\infty \\
			&\le
			M_0
			\Bigl\|
			\frac{1}{W}\sum_{i=1}^W
			X_{i,S}X_i^\top\Delta
			\Bigr\|_\infty \\
			&=
			M_0
			\Bigl\|
			\frac{1}{W}X_S^\top X\Delta
			\Bigr\|_\infty.
		\end{align*}
	Since \(\operatorname{supp}(\Delta)\subseteq J\),
\[
\frac{1}{W}X_S^\top X\Delta
=
\frac{1}{W}X_S^\top X_J\Delta_J.
\]
By {\normalfont (A4)},
\[
\left\|
\frac{1}{W}X_S^\top X_J
\right\|_{\infty\to\infty}
\le C_\infty.
\]
Therefore,
\[
\left\|
\mathbb E[(g_t(\beta)-g_t(\beta^\ast))_S\mid X]
\right\|_\infty
\le
M_1C_\infty\|\Delta_J\|_1
\le
M_1C_\infty(1+c)\|\Delta_S\|_1.
\]
		Hence
		\[
		\Bigl\|
		\frac{1}{W}X_S^\top X\Delta
		\Bigr\|_\infty
		\le
		\Bigl\|\frac{1}{W}X_S^\top X_S\Bigr\|_{\infty\to\infty}\|\Delta_S\|_1
		+
		\Bigl\|\frac{1}{W}X_S^\top X_{S^c}\Bigr\|_{\infty\to\infty}\|\Delta_{S^c}\|_1.
		\]
		Since $|S|=s$ is fixed and $X_i$ is sub-Gaussian, standard concentration for the
		finite active block implies that with probability at least $1-\delta/3$,
		\[
		\Bigl\|\frac{1}{W}X_S^\top X_S\Bigr\|_{\infty\to\infty}\le C_{SS},
		\qquad
		\Bigl\|\frac{1}{W}X_S^\top X_{S^c}\Bigr\|_{\infty\to\infty}\le C_{S S^c},
		\]
		for finite constants $C_{SS},C_{S S^c}>0$ depending only on $K$ and $s$.
		Using the cone condition $\|\Delta_{S^c}\|_1\le c\|\Delta_S\|_1$, we get
		\begin{equation}
			\label{eq:on-support-condmean}
			\Bigl\|
			\mathbb E[(g_t(\beta)-g_t(\beta^\ast))_S\mid X]
			\Bigr\|_\infty
			\le
			L_0\,\|\Delta_S\|_1,
			\qquad
			L_0:=M_0(C_{SS}+cC_{S S^c}).
		\end{equation}
		For each $j\in S$, define
		\[
		U_{ij}
		:=
		X_{ij}
		\Bigl(
		\zeta_i(\Delta)-\mathbb E[\zeta_i(\Delta)\mid X_i]
		\Bigr).
		\]
		Then
		\[
		(g_t(\beta)-g_t(\beta^\ast))_S
		-
		\mathbb E[(g_t(\beta)-g_t(\beta^\ast))_S\mid X]
		=
		\frac{1}{W}\sum_{i=1}^W U_i,
		\]
		where $U_i=(U_{ij})_{j\in S}$.
		For each fixed $j\in S$, the variables $\{U_{ij}\}_{i=1}^W$ are independent,
		centered given $X$, and satisfy
		\[
		|\zeta_i(\Delta)-\mathbb E[\zeta_i(\Delta)\mid X_i]|\le 2.
		\]
		Since $X_{ij}$ is sub-Gaussian, each $U_{ij}$ is centered sub-Gaussian with parameter
		bounded by a constant multiple of $K$. Therefore there exists $c_0>0$ such that
		for any $u>0$,
		\[
		\mathbb P\!\left(
		\left|
		\frac{1}{W}\sum_{i=1}^W U_{ij}
		\right|
		\ge u
		\right)
		\le
		2\exp\!\left(-\frac{c_0Wu^2}{K^2}\right).
		\]
		Applying a union bound over the $s$ support coordinates and choosing
		\[
		u=C_0\sqrt{\frac{\log(p/\delta)}{W}}
		\]
		with $C_0$ sufficiently large yields, with probability at least $1-\delta/3$,
		\begin{equation}
			\label{eq:on-support-fluctuation}
			\Bigl\|
			(g_t(\beta)-g_t(\beta^\ast))_S
			-
			\mathbb E[(g_t(\beta)-g_t(\beta^\ast))_S\mid X]
			\Bigr\|_\infty
			\le
			C_0\sqrt{\frac{\log(p/\delta)}{W}}.
		\end{equation}
		Combining \eqref{eq:on-support-condmean} and \eqref{eq:on-support-fluctuation}, we obtain
		with probability at least $1-2\delta/3$,
		\[
		\|(g_t(\beta)-g_t(\beta^\ast))_S\|_\infty
		\le
		L_0\,\|\Delta_S\|_1
		+
		C_0\sqrt{\frac{\log(p/\delta)}{W}}.
		\]
		Together with \eqref{eq:on-support-truth} and \eqref{eq:on-support-triangle}, this gives
		\[
		\|g_{t,S}(\beta)\|_\infty
		\le
		L_g\,\|\Delta_S\|_1
		+
		C_0\sqrt{\frac{\log(p/\delta)}{W}}
		\]
		with probability at least $1-\delta$, after enlarging constants if necessary. Here
		$L_g>0$ depends only on $K$, $s$, $(m_0,M_0)$, $c$, and $R$.
		This proves the claim.
	\end{proof}
	\begin{Blemma}[Candidate-restricted \(\ell_1\) leakage]
\label{app:lem:candidate-l1-leakage}
Suppose that the coordinatewise off-support leakage bound
\[
\|\bar g_{t,S^c}(\beta_t)\|_\infty
\le
\rho\|\Delta_{t,S}\|_1+\sigma_W
\]
and the coordinatewise on-support bound
\[
\|\bar g_{t,S}(\beta_t)\|_\infty
\le
L_g\|\Delta_{t,S}\|_1+\sigma_W
\]
hold on an event \(\mathcal E_\infty\), where
\[
\sigma_W=C_0\sqrt{\frac{\log(p/\delta)}{W}}.
\]
Let \(d_t=P_{A_j}\bar g_t(\beta_t)\) with \(|A_j|\le s+m\). Then on \(\mathcal E_\infty\),
\[
\|d_{t,S^c}\|_1
\le
(s+m)\rho\|\Delta_{t,S}\|_1+(s+m)\sigma_W,
\]
and
\[
\|d_{t,S}\|_1
\le
s_0L_g\|\Delta_{t,S}\|_1+s_0\sigma_W.
\]
Equivalently, with \(\zeta_W=\sigma_W\),
\[
\|d_{t,S^c}\|_1
\le
L_{\mathrm{off}}\{\|\Delta_{t,S}\|_1+\zeta_W\},
\qquad
\|d_{t,S}\|_1
\le
L_{\mathrm{on}}\{\|\Delta_{t,S}\|_1+\zeta_W\},
\]
where one may take
\[
L_{\mathrm{off}}=(s+m)\max\{\rho,1\},
\qquad
L_{\mathrm{on}}=s_0\max\{L_g,1\}.
\]
In particular, \(L_{\mathrm{off}}\) and \(L_{\mathrm{on}}\) depend on the working sparsity and candidate size, but not on the ambient dimension \(p\), except through the logarithmic factor in \(\zeta_W\).
\end{Blemma}
\begin{proof}
Since \(d_t=P_{A_j}\bar g_t(\beta_t)\),
\[
\|d_{t,S^c}\|_1
=
\|\{\bar g_t(\beta_t)\}_{A_j\cap S^c}\|_1
\le
|A_j\cap S^c|\,\|\bar g_{t,S^c}(\beta_t)\|_\infty.
\]
Because \(|A_j\cap S^c|\le |A_j|\le s+m\), the off-support bound follows from the coordinatewise leakage inequality. Similarly,
\[
\|d_{t,S}\|_1
\le
|S|\,\|\bar g_{t,S}(\beta_t)\|_\infty
=
s_0\|\bar g_{t,S}(\beta_t)\|_\infty,
\]
which gives the on-support bound. The inflated-cone form follows by taking \(\zeta_W=\sigma_W\) and enlarging constants.
\end{proof}

\subsection*{Cone invariance}
\begin{proof}[Proof of Theorem~\ref{thm:cone-invariance}]
Work on the gradient-control event described in Section~\ref{subsec:assumptions-curvature}. Fix an epoch \(j\ge J_0\). To avoid ambiguity at the final thresholding step, write \(z_t\) for the pre-threshold trajectory within this epoch:
\[
z_{\tau_j}=\beta_{\tau_j},
\qquad
z_{t+1}=z_t-\eta_t d_t,
\qquad
t\in\mathcal T_j.
\]
Thus \(z_t=\beta_t\) for all non-final times inside the epoch, and
\(z_{\tau_{j+1}}=\widetilde\beta_{\tau_{j+1}}\) is the provisional endpoint before hard thresholding. Let \(e_t=z_t-\beta^\ast\). We prove that
\[
e_t\in\Gamma_\zeta(c_{\mathrm{large}},S),
\qquad
\|e_t\|_2\le R,
\qquad
t=\tau_j,\ldots,\tau_{j+1}.
\]

At the start of the epoch, the post-burn-in condition gives
\[
e_{\tau_j}=\Delta_{\tau_j}\in\Gamma_\zeta(c_{\mathrm{small}},S)
\subseteq
\Gamma_\zeta(c_{\mathrm{large}},S),
\qquad
\|e_{\tau_j}\|_2\le R_{\mathrm{in}}<R.
\]

Define
\[
a_t=\|e_{t,S^c}\|_1,
\qquad
b_t=\|e_{t,S}\|_1,
\qquad
D_t=b_t+\zeta_W,
\qquad
r_t=\frac{a_t}{D_t}.
\]
Then \(e_t\in\Gamma_\zeta(c,S)\) is equivalent to \(r_t\le c\).

We first derive the ratio recursion. Suppose that \(e_t\in\Gamma_\zeta(c_{\mathrm{large}},S)\) and \(\|e_t\|_2\le R\). The update \(e_{t+1}=e_t-\eta_t d_t\) gives, by the triangle inequality and the gradient guards,
\[
a_{t+1}
=
\|e_{t+1,S^c}\|_1
\le
a_t+\eta_t\|d_{t,S^c}\|_1
\le
a_t+\eta_tL_{\mathrm{off}}D_t.
\]
For the on-support component, the reverse triangle inequality gives
\[
b_{t+1}
=
\|e_{t+1,S}\|_1
\ge
b_t-\eta_t\|d_{t,S}\|_1
\ge
b_t-\eta_tL_{\mathrm{on}}D_t.
\]
Therefore
\[
D_{t+1}
=
b_{t+1}+\zeta_W
\ge
(1-\eta_tL_{\mathrm{on}})D_t.
\]
Since \(\eta_tL_{\mathrm{on}}\le1/2\), the denominator is positive and
\[
r_{t+1}
=
\frac{a_{t+1}}{D_{t+1}}
\le
\frac{r_t+\eta_tL_{\mathrm{off}}}
{1-\eta_tL_{\mathrm{on}}}.
\]
Using \(1/(1-x)\le 1+2x\) for \(0\le x\le1/2\), we obtain
\[
r_{t+1}
\le
(r_t+\eta_tL_{\mathrm{off}})(1+2\eta_tL_{\mathrm{on}}).
\]
As long as \(r_t\le c_{\mathrm{large}}\), this implies
\[
r_{t+1}
\le
r_t
+
2\eta_t\{L_{\mathrm{off}}+c_{\mathrm{large}}L_{\mathrm{on}}\}.
\]
Indeed, the additional term \(2\eta_t^2L_{\mathrm{off}}L_{\mathrm{on}}\) is bounded by
\(\eta_tL_{\mathrm{off}}\) because \(\eta_tL_{\mathrm{on}}\le1/2\), and is absorbed into the displayed bound.

Summing the recursion over the epoch gives, for every \(t\le\tau_{j+1}\) before a possible exit from the large cone,
\[
r_t
\le
c_{\mathrm{small}}
+
2\{L_{\mathrm{off}}+c_{\mathrm{large}}L_{\mathrm{on}}\}
\sum_{u=\tau_j}^{t-1}\eta_u .
\]
Using the epoch-mass condition,
\[
\sum_{u\in\mathcal T_j}\eta_u
\le
\frac{c_{\mathrm{large}}-c_{\mathrm{small}}}
{2\{L_{\mathrm{off}}+c_{\mathrm{large}}L_{\mathrm{on}}\}},
\]
we obtain \(r_t\le c_{\mathrm{large}}\). Thus the pre-threshold trajectory cannot exit the enlarged inflated cone during the epoch.

It remains to check that the trajectory stays inside the local radius \(R\). While \(e_t\in\Gamma_\zeta(c_{\mathrm{large}},S)\) and \(\|e_t\|_2\le R\), the gradient guards imply
\[
\|d_t\|_2
\le
\|d_t\|_1
\le
\|d_{t,S}\|_1+\|d_{t,S^c}\|_1
\le
(L_{\mathrm{on}}+L_{\mathrm{off}})D_t.
\]
Since \(D_t=\|e_{t,S}\|_1+\zeta_W\le \sqrt{s_0}\|e_t\|_2+\zeta_W\le \sqrt{s_0}R+\zeta_W\), we have
\[
\|d_t\|_2
\le
(L_{\mathrm{on}}+L_{\mathrm{off}})(\sqrt{s_0}R+\zeta_W).
\]
Therefore, for any \(t\le\tau_{j+1}\),
\[
\|e_t\|_2
\le
\|e_{\tau_j}\|_2
+
\sum_{u=\tau_j}^{t-1}\eta_u\|d_u\|_2
\le
R_{\mathrm{in}}
+
B_j(L_{\mathrm{on}}+L_{\mathrm{off}})(\sqrt{s_0}R+\zeta_W).
\]
By the local-radius part of the epoch-mass condition, the right-hand side is at most \(R\). Hence the trajectory also remains inside the local ball.

Combining the cone-ratio argument and the radius argument proves that
\[
e_t\in\Gamma_\zeta(c_{\mathrm{large}},S),
\qquad
\|e_t\|_2\le R,
\qquad
t=\tau_j,\ldots,\tau_{j+1}.
\]
Since \(z_t=\beta_t\) for \(t\in\mathcal T_j\) before the final thresholding step and \(z_{\tau_{j+1}}=\widetilde\beta_{\tau_{j+1}}\), the theorem follows.
\end{proof}

\subsection*{Convergence}

\begin{Blemma}[Support-preserving thresholding]
\label{app:lem:support-preserving-ht}
Let \(v\in\mathbb R^p\), and suppose \(S\subseteq\operatorname{supp}\{H_s(v)\}\). Then
\[
\|H_s(v)-\beta^\ast\|_2
\le
\|v-\beta^\ast\|_2.
\]
Moreover, if \(\beta_{\min}=\min_{j\in S}|\beta_j^\ast|>0\) and
\[
\|v-\beta^\ast\|_2<\frac{\beta_{\min}}{2},
\]
then \(S\subseteq\operatorname{supp}\{H_s(v)\}\) whenever \(s\ge s_0\).
\end{Blemma}

\begin{proof}
Let \(T=\operatorname{supp}\{H_s(v)\}\). If \(S\subseteq T\), then \(\beta^\ast_{T^c}=0\), and therefore
\[
\|H_s(v)-\beta^\ast\|_2^2
=
\|(v-\beta^\ast)_T\|_2^2
\le
\|v-\beta^\ast\|_2^2.
\]

For the beta-min claim, if \(j\in S\), then
\[
|v_j|\ge |\beta_j^\ast|-|v_j-\beta_j^\ast|
>
\beta_{\min}-\frac{\beta_{\min}}{2}
=
\frac{\beta_{\min}}{2}.
\]
If \(j\notin S\), then \(\beta_j^\ast=0\), so
\[
|v_j|=|v_j-\beta_j^\ast|
<
\frac{\beta_{\min}}{2}.
\]
Thus every coordinate in \(S\) has magnitude larger than every coordinate in \(S^c\). Since \(s\ge s_0=|S|\), the hard-thresholded support must contain \(S\).
\end{proof}

\begin{proof}[Proof of Theorem~\ref{thm:two-phase-convergence}]
Work on the local high-probability event described in Section~\ref{subsec:assumptions-curvature}. Let
\(\Delta_t=\beta_t-\beta^\ast\). During epoch \(j\), AIHT uses the candidate-restricted direction
\(d_t=P_{A_j}\bar g_t(\beta_t)\). By post-burn-in support coverage, \(S\subseteq A_j\), and since the iterate is updated only on \(A_j\), we have \(\operatorname{supp}(\Delta_t)\subseteq A_j\). Hence
\[
\langle d_t,\Delta_t\rangle
=
\langle \bar g_t(\beta_t),\Delta_t\rangle .
\]

Consider first a non-thresholded provisional step:
\[
\widetilde\Delta_{t+1}
=
\Delta_t-\eta_t d_t .
\]
Expanding the squared norm gives
\[
\|\widetilde\Delta_{t+1}\|_2^2
=
\|\Delta_t\|_2^2
-
2\eta_t\langle d_t,\Delta_t\rangle
+
\eta_t^2\|d_t\|_2^2 .
\]
Using the support-coverage identity above,
\[
\langle d_t,\Delta_t\rangle
=
\bigl\langle
\bar g_t(\beta_t)-\bar g_t(\beta^\ast),\Delta_t
\bigr\rangle
+
\bigl\langle
\bar g_t(\beta^\ast),\Delta_t
\bigr\rangle .
\]
The local restricted curvature gives
\[
\bigl\langle
\bar g_t(\beta_t)-\bar g_t(\beta^\ast),\Delta_t
\bigr\rangle
\ge
\mu\|\Delta_t\|_2^2.
\]
The score-at-truth control gives
\[
-2\eta_t
\bigl\langle
\bar g_t(\beta^\ast),\Delta_t
\bigr\rangle
\le
\frac{\mu}{2}\eta_t\|\Delta_t\|_2^2
+
C_{\mathrm{sc}}\eta_t^2.
\]
Together with \(\|d_t\|_2\le G_{\bar s}\), these inequalities imply
\[
\|\widetilde\Delta_{t+1}\|_2^2
\le
\left(1-\frac{3\mu}{2}\eta_t\right)\|\Delta_t\|_2^2
+
(C_{\mathrm{sc}}+G_{\bar s}^2)\eta_t^2.
\]
After reducing the contraction constant and enlarging \(C\), we obtain the simpler recursion
\begin{equation}
\label{eq:basic-local-recursion}
\|\widetilde\Delta_{t+1}\|_2^2
\le
(1-\mu\eta_t)\|\Delta_t\|_2^2
+
C\eta_t^2 .
\end{equation}

If \(t\) is not an epoch endpoint, then \(\Delta_{t+1}=\widetilde\Delta_{t+1}\). If \(t\) is an epoch endpoint, the support-preserving thresholding condition gives
\[
\|\Delta_{t+1}\|_2
=
\|H_s(\widetilde\beta_{t+1})-\beta^\ast\|_2
\le
\|\widetilde\beta_{t+1}-\beta^\ast\|_2
=
\|\widetilde\Delta_{t+1}\|_2 .
\]
Thus the recursion
\begin{equation}
\label{eq:post-threshold-local-recursion}
\|\Delta_{t+1}\|_2^2
\le
(1-\mu\eta_t)\|\Delta_t\|_2^2
+
C\eta_t^2
\end{equation}
holds for every post-burn-in step.

For Phase~I, iterate \eqref{eq:post-threshold-local-recursion} over an epoch \(\mathcal T_j\). Using \(1-x\le e^{-x}\), we get
\[
\|\Delta_{\tau_{j+1}}\|_2^2
\le
\exp\!\left(-\mu\sum_{t\in\mathcal T_j}\eta_t\right)
\|\Delta_{\tau_j}\|_2^2
+
C\sum_{t\in\mathcal T_j}\eta_t^2,
\]
after enlarging \(C\) to absorb the products of contraction factors multiplying the tolerance terms. This proves \eqref{eq:phaseI-epoch-recursion}. If \(B_j=\sum_{t\in\mathcal T_j}\eta_t\ge B_{\min}\), then
\[
a_j^{1/2}
\le
\exp(-\mu B_{\min}/2)
=:\lambda<1,
\]
which gives the displayed epochwise contraction in norm.

For Phase~II, let \(e_t=\|\Delta_t\|_2^2\) and use
\[
\eta_t=\frac{1}{c_{\mathrm{eff}}(t+b_2)},
\qquad
0<c_{\mathrm{eff}}\le \mu.
\]
Then \eqref{eq:post-threshold-local-recursion} implies
\[
e_{t+1}
\le
\left(1-\frac{\mu}{c_{\mathrm{eff}}(t+b_2)}\right)e_t
+
\frac{C}{(t+b_2)^2}.
\]
Let \(a=\mu/c_{\mathrm{eff}}\ge1\) and \(n_t=t+b_2\). Then
\[
e_{t+1}\le \left(1-\frac{a}{n_t}\right)e_t+\frac{C}{n_t^2}.
\]
Multiplying by \(n_t+1\) gives
\[
(n_t+1)e_{t+1}
\le
(n_t+1)\left(1-\frac{a}{n_t}\right)e_t
+
\frac{C(n_t+1)}{n_t^2}.
\]
Since \(a\ge1\),
\[
(n_t+1)\left(1-\frac{a}{n_t}\right)
=
n_t+1-a-\frac{a}{n_t}
\le
n_t.
\]
Therefore
\[
(n_t+1)e_{t+1}
\le
n_t e_t+\frac{C}{n_t}.
\]
Summing from \(T_0+1\) to \(t\) yields
\[
(t+b_2)e_t
\le
C+
C\sum_{u=T_0+1}^{t}\frac1{u+b_2}
\le
C\log(t+b_2).
\]
Dividing by \(t+b_2\) proves
\[
\|\Delta_t\|_2^2
\le
C\frac{\log(t+b_2)}{t+b_2}.
\]
\end{proof}
\subsection*{Regret Bound}
\begin{proof}[Proof of Theorem~\ref{thm:regret-aiht}]
Work on the same local high-probability event as in Theorem~\ref{thm:two-phase-convergence}. We first bound the refinement phase \(t\ge T_0+1\). By the local one-point curvature condition,
\[
Q_t(\beta_t)-Q_t(\beta^\ast)
\le
\langle \bar g_t(\beta_t),\Delta_t\rangle
-
\frac{\mu}{2}\|\Delta_t\|_2^2.
\]
Because \(\operatorname{supp}(\Delta_t)\subseteq A_j\), we have
\[
\langle \bar g_t(\beta_t),\Delta_t\rangle
=
\langle d_t,\Delta_t\rangle .
\]

The candidate-restricted update and support-preserving thresholding imply
\[
\|\Delta_{t+1}\|_2^2
\le
\|\Delta_t-\eta_t d_t\|_2^2
=
\|\Delta_t\|_2^2
-
2\eta_t\langle d_t,\Delta_t\rangle
+
\eta_t^2\|d_t\|_2^2.
\]
Rearranging and using \(\|d_t\|_2\le G_{\bar s}\), we obtain
\[
\langle d_t,\Delta_t\rangle
\le
\frac{\|\Delta_t\|_2^2-\|\Delta_{t+1}\|_2^2}{2\eta_t}
+
\frac{\eta_t}{2}G_{\bar s}^2.
\]
Therefore
\[
Q_t(\beta_t)-Q_t(\beta^\ast)
\le
\left(\frac1{2\eta_t}-\frac{\mu}{2}\right)\|\Delta_t\|_2^2
-
\frac1{2\eta_t}\|\Delta_{t+1}\|_2^2
+
\frac{\eta_t}{2}G_{\bar s}^2.
\]

In Phase~II, \(\eta_t=1/\{c_{\mathrm{eff}}(t+b_2)\}\) with \(c_{\mathrm{eff}}\le\mu\). Hence
\[
\frac1{2\eta_t}-\frac{\mu}{2}
=
\frac{c_{\mathrm{eff}}(t+b_2)}{2}-\frac{\mu}{2}
\le
\frac{c_{\mathrm{eff}}(t+b_2-1)}{2}.
\]
Thus
\[
Q_t(\beta_t)-Q_t(\beta^\ast)
\le
\frac{c_{\mathrm{eff}}}{2}
\left\{
(t+b_2-1)\|\Delta_t\|_2^2
-
(t+b_2)\|\Delta_{t+1}\|_2^2
\right\}
+
\frac{G_{\bar s}^2}{2c_{\mathrm{eff}}(t+b_2)}.
\]
Summing from \(t=T_0+1\) to \(T\), the first term telescopes and gives a finite constant depending on the local error at \(T_0+1\). The second term gives
\[
\sum_{t=T_0+1}^{T}
\frac{G_{\bar s}^2}{2c_{\mathrm{eff}}(t+b_2)}
\le
\frac{G_{\bar s}^2}{2c_{\mathrm{eff}}}
\log\!\left(\frac{T+b_2}{T_0+1+b_2}\right).
\]

It remains to account for the discovery phase. For \(t\le T_0\), convexity and the same update identity give the standard online-gradient bound
\[
Q_t(\beta_t)-Q_t(\beta^\ast)
\le
\frac{\|\Delta_t\|_2^2-\|\Delta_{t+1}\|_2^2}{2\eta_t}
+
\frac{\eta_t}{2}G_{\bar s}^2,
\]
on the local part of the trajectory. With \(\eta_t\asymp t^{-1/2}\), bounded local radius, and bounded candidate-restricted directions, summing this inequality gives
\[
C_{\mathrm I}(T_0)=O(\sqrt{T_0}).
\]
Any finite cost before the local post-burn-in event is collected into \(C_{\mathrm{burn}}\). Combining the burn-in, discovery, and refinement contributions proves \eqref{eq:regret-main}.
\end{proof}

\section{Proofs for Section~\ref{sec:piecewise}}
\label{app:proof-section4}

\begin{proof}[Proof of Lemma~\ref{lem:score-detector-concentration}]
Fix \(t\le T\) and consider one of the two windows \(\mathcal I_t^\pm\). Suppose it is contained in segment \(k\). The reference point \(\breve\beta_t=\beta_{t-2h}\) is measurable with respect to the sigma-field generated by observations up to time \(t-2h\), while both detection windows start after time \(t-2h\). Conditional on this past sigma-field, the observations in the detection window are independent of \(\breve\beta_t\).

For coordinate \(j\), define
\[
Z_{ij}(\breve\beta_t)
=
-X_{ij}\{\tau-\mathbf 1(Y_i\le X_i^\top\breve\beta_t)\}
-
[\Psi_k(\breve\beta_t)]_j .
\]
Conditional on \(\breve\beta_t\), the variables \(Z_{ij}(\breve\beta_t)\) are independent and centered. Since the factor \(\tau-\mathbf 1(Y_i\le X_i^\top\breve\beta_t)\) is bounded and \(X_{ij}\) is sub-Gaussian under the segmentwise design condition, \(Z_{ij}(\breve\beta_t)\) is sub-Gaussian with a parameter bounded by a constant depending only on the common segmentwise design constant. This bound is uniform in \(\breve\beta_t\).

Therefore, for some constants \(c>0\) and \(C>0\),
\[
\mathbb P\left(
\left|
[\bar g_t^\pm(\breve\beta_t)-\Psi_k(\breve\beta_t)]_j
\right|
>u
\;\middle|\;
\breve\beta_t
\right)
\le
2\exp(-ch u^2).
\]
Taking \(u=\lambda_h/2\), choosing \(C_D\) sufficiently large, and applying a union bound over the two signs, all coordinates \(j=1,\ldots,p\), and all \(t\le T\), gives the simultaneous event with probability at least \(1-\delta\).
\end{proof}

\begin{proof}[Proof of Lemma~\ref{lem:no-false-clean}]
Suppose both adjacent windows are contained in the same segment \(k\). On the event of Lemma~\ref{lem:score-detector-concentration},
\[
\|\bar g_t^+(\breve\beta_t)-\Psi_k(\breve\beta_t)\|_\infty
\le
\frac{\lambda_h}{2},
\qquad
\|\bar g_t^-(\breve\beta_t)-\Psi_k(\breve\beta_t)\|_\infty
\le
\frac{\lambda_h}{2}.
\]
By the triangle inequality,
\[
D_t
=
\|\bar g_t^+(\breve\beta_t)-\bar g_t^-(\breve\beta_t)\|_\infty
\le
\lambda_h
<
b_h,
\]
because \(b_h=2\lambda_h\). Hence no restart occurs on clean adjacent windows.
\end{proof}

\begin{proof}[Proof of Lemma~\ref{lem:detect-change}]
At \(t_k=\nu_k+h-1\), the left and right detection windows are
\[
\mathcal I_{t_k}^-=\{\nu_k-h,\ldots,\nu_k-1\},
\qquad
\mathcal I_{t_k}^+=\{\nu_k,\ldots,\nu_k+h-1\}.
\]
Thus the left window is contained in segment \(k-1\), while the right window is contained in segment \(k\).

By the triangle inequality and the Lipschitz condition in {\normalfont (D2)},
\[
\begin{aligned}
&\|\Psi_k(\breve\beta_{t_k})-\Psi_{k-1}(\breve\beta_{t_k})\|_\infty \\
&\quad\ge
\|\Psi_k(\theta_{k-1})-\Psi_{k-1}(\theta_{k-1})\|_\infty
-
\|\Psi_k(\breve\beta_{t_k})-\Psi_k(\theta_{k-1})\|_\infty
-
\|\Psi_{k-1}(\breve\beta_{t_k})-\Psi_{k-1}(\theta_{k-1})\|_\infty \\
&\quad\ge
\mathfrak J_k
-
2L_{\mathrm{det}}\|\breve\beta_{t_k}-\theta_{k-1}\|_2
\ge
\mathfrak J_k-2L_{\mathrm{det}}r_{\mathrm{det}}.
\end{aligned}
\]
By the score-jump separation in {\normalfont (D2)}, the last expression is at least \(4\lambda_h\). Hence
\[
\|\Psi_k(\breve\beta_{t_k})-\Psi_{k-1}(\breve\beta_{t_k})\|_\infty
\ge
4\lambda_h .
\]

On the event of Lemma~\ref{lem:score-detector-concentration},
\[
\|\bar g_{t_k}^+(\breve\beta_{t_k})-\Psi_k(\breve\beta_{t_k})\|_\infty
\le
\frac{\lambda_h}{2},
\qquad
\|\bar g_{t_k}^-(\breve\beta_{t_k})-\Psi_{k-1}(\breve\beta_{t_k})\|_\infty
\le
\frac{\lambda_h}{2}.
\]
Therefore,
\[
\begin{aligned}
D_{t_k}
&=
\|\bar g_{t_k}^+(\breve\beta_{t_k})
-
\bar g_{t_k}^-(\breve\beta_{t_k})\|_\infty \\
&\ge
\|\Psi_k(\breve\beta_{t_k})
-
\Psi_{k-1}(\breve\beta_{t_k})\|_\infty
-
\lambda_h \\
&\ge
3\lambda_h
>
2\lambda_h
=
b_h.
\end{aligned}
\]
Thus the detector triggers at \(t_k\), so the hard restart occurs at \(t_k+1=\nu_k+h\). If the detector triggers earlier on a contaminated adjacent-window pair, the restart occurs even sooner. Hence the changepoint is detected no later than \(\nu_k+h\).
\end{proof}

\begin{proof}[Proof of Theorem~\ref{thm:segmentwise-restart}]
Work on the intersection of the detector concentration event from Lemma~\ref{lem:score-detector-concentration} and the segmentwise stationary high-probability events in Assumption~\textnormal{(D1)}. By the choice of confidence budgets and a union bound, this event has probability at least \(1-\delta\).

We first prove the restart-delay claim. The argument proceeds segment by segment. Before changepoint \(\nu_k\), any adjacent detection windows fully contained in segment \(k-1\) are clean. By Lemma~\ref{lem:no-false-clean}, such windows do not trigger a restart. Hence the restart associated with segment \(k\) cannot occur before \(\nu_k\).

It remains to show that a restart occurs no later than \(\nu_k+h\). By the induction hypothesis for the previous segment, the restart associated with segment \(k-1\) occurs no later than \(\nu_{k-1}+h\). Assumption~\textnormal{(D3)} gives
\[
\nu_k-\nu_{k-1}
\ge
T_{\mathrm{rec}}(r_{\mathrm{det}})+3h.
\]
Therefore, at the frozen-reference time \(t_k-2h=\nu_k-h-1\),
the local age since the previous restart is at least
\[
(\nu_k-h-1)-(\nu_{k-1}+h)+1
=
\nu_k-\nu_{k-1}-2h
\ge
T_{\mathrm{rec}}(r_{\mathrm{det}})+h.
\]
Thus the previous segment has had enough clean observations for recovery, and Assumption~\textnormal{(D1)} yields
\[
\|\breve\beta_{t_k}-\theta_{k-1}\|_2
=
\|\beta_{t_k-2h}-\theta_{k-1}\|_2
\le
r_{\mathrm{det}}.
\]
Lemma~\ref{lem:detect-change} then gives \(D_{t_k}>b_h\), so the restart occurs no later than \(t_k+1=\nu_k+h\). Hence
\[
\nu_k\le \widehat\nu_k\le \nu_k+h,
\]
which proves part \textnormal{(i)}.

For segmentwise recovery, after the restart at \(\widehat\nu_k\), the algorithm sets \(\beta_{\widehat\nu_k}=0\), flushes stale observations, resets the local clock, returns to Phase~I, and rebuilds candidate sets using only observations from the new segment. Once \(T_{\mathrm{rec}}(r_{\mathrm{det}})\) clean observations have accumulated, Assumption~\textnormal{(D1)} ensures that the local post-burn-in conditions of Section~\ref{sec:main-theory} hold with target \(\theta_k\) and support \(S_k\). Applying Theorem~\ref{thm:two-phase-convergence} to the clean local run gives, during the local Phase~II refinement regime,
\[
\|\beta_t-\theta_k\|_2^2
\le
C\frac{\log(a_t^{(k)}+b_2)}{a_t^{(k)}+b_2},
\qquad
a_t^{(k)}=t-\widehat\nu_k+1.
\]
This proves part \textnormal{(ii)}.

It remains to bound dynamic regret. Decompose the contribution on each segment into three parts.

First, after each changepoint \(\nu_k\), part \textnormal{(i)} shows that at most \(h\) observations are processed before the restart associated with the new segment. By Assumption~\textnormal{(D1)}, the clean sliding-window excess loss of these iterates is bounded by \(M_{\mathrm{loc}}\) per round on the high-probability event. Thus the total detection-delay cost is at most \(KhM_{\mathrm{loc}}\).

Second, after each restart, the local buffer refill, burn-in, support-discovery period, and transition to the post-burn-in local regime contribute at most \(\mathcal R_{\mathrm I}\) by the segmentwise guarantee in Assumption~\textnormal{(D1)}. Across \(K+1\) segments, this contributes \((K+1)\mathcal R_{\mathrm I}\).

Third, during the local Phase~II refinement regime on segment \(k\), Theorem~\ref{thm:regret-aiht}, applied with \(\theta_k\) and \(S_k\), gives
\[
\mathcal R_{\mathrm{II},k}
\le
\frac{G_{\bar s}^2}{2c_{\mathrm{eff}}}
\log\{1+\nu_{k+1}-\widehat\nu_k\}.
\]
Summing the delay, discovery, and refinement contributions over all segments gives
\[
\mathcal R_T^{\mathrm{dyn}}
\le
KhM_{\mathrm{loc}}
+
(K+1)\mathcal R_{\mathrm I}
+
\frac{G_{\bar s}^2}{2c_{\mathrm{eff}}}
\sum_{k=0}^{K}
\log\{1+\nu_{k+1}-\widehat\nu_k\},
\]
which proves \eqref{eq:dynamic-regret-main}. Since \(\nu_{k+1}-\widehat\nu_k\le T\), the simplified bound \eqref{eq:dynamic-regret-simplified} follows immediately.
\end{proof}

\end{document}